\documentclass[letterpaper]{article} % DO NOT CHANGE THIS
\usepackage{aaai2026}  % DO NOT CHANGE THIS
\usepackage{times}  % DO NOT CHANGE THIS
\usepackage{helvet}  % DO NOT CHANGE THIS
\usepackage{courier}  % DO NOT CHANGE THIS
\usepackage[hyphens]{url}  % DO NOT CHANGE THIS
\usepackage{graphicx} % DO NOT CHANGE THIS
\usepackage{natbib}  % DO NOT CHANGE THIS AND DO NOT ADD ANY OPTIONS TO IT
\usepackage{caption} % DO NOT CHANGE THIS AND DO NOT ADD ANY OPTIONS TO IT
\usepackage{algorithm}
\usepackage{algorithmic}
\usepackage{amsmath}
\usepackage{amsmath}
\usepackage{amssymb}
\usepackage{mathtools}
\usepackage{amsthm}
\usepackage{pifont}

\usepackage{microtype}
\usepackage{color}
\usepackage{pifont}
\usepackage{booktabs}
\usepackage{multirow}
\usepackage{array} 
\usepackage{makecell}

\usepackage{microtype}
\usepackage{booktabs} % for professional tables
\usepackage{graphicx}  % for rotatebox
\usepackage{multirow}  % for multirow cells in tables
\usepackage{amsmath} 
\usepackage{subfig}

\def\ie{\emph{i.e.}}

\def\etal{\emph{et al.}}

\newcommand{\ours}{\text{CAE}}
\newcommand{\x}{\mathbf{x}}

\newcommand{\Ev}{\Phi_{\mathsf{v}}}
\newcommand{\Et}{\Phi_{\mathsf{t}}}
\newcommand{\bu}{\mathbf{u}}
\newcommand{\bv}{\mathbf{v}}

\newcommand{\tableCellHeight}{1}
\newcommand{\tabstyle}[1]{
  \setlength{\tabcolsep}{#1}
  \renewcommand{\arraystretch}{\tableCellHeight}
  \centering
  \small
}

\newtheorem{theorem}{Theorem}
\usepackage{bm}
\usepackage{color}
\usepackage{xcolor}
\usepackage{newfloat}
\usepackage{listings}
\DeclareCaptionStyle{ruled}{labelfont=normalfont,labelsep=colon,strut=off} % DO NOT CHANGE THIS
\floatstyle{ruled}
\newfloat{listing}{tb}{lst}{}
\floatname{listing}{Listing}
\title{Hierarchical Semantic Alignment for Image Clustering}
\author{
    Written by AAAI Press Staff\textsuperscript{\rm 1}\thanks{With help from the AAAI Publications Committee.}\\
    AAAI Style Contributions by Pater Patel Schneider,
    Sunil Issar,\\
    J. Scott Penberthy,
    George Ferguson,
    Hans Guesgen,
    Francisco Cruz\equalcontrib,
    Marc Pujol-Gonzalez\equalcontrib
}

\author{
    Xingyu Zhu\textsuperscript{\rm 1, 2}, 
    Beier Zhu\textsuperscript{\rm 2}, 
    Yunfan Li\textsuperscript{\rm 3}, 
    Junfeng Fang\textsuperscript{\rm 4}, 
    Shuo Wang\textsuperscript{\rm 1}\thanks{Corresponding author}, \\
    Kesen Zhao\textsuperscript{\rm  2},
    Hanwang Zhang\textsuperscript{\rm  2}
}

\affiliations{
    \textsuperscript{\rm 1} University of Science and Technology of China;
    \textsuperscript{\rm 2} Nanyang Technological University;\\
    \textsuperscript{\rm 3} Sichuan University;
    \textsuperscript{\rm 4} National University of Singapore.\\
    xingyuzhu@mail.ustc.edu.cn, shuowang.edu@gmail.com, \\
}

\usepackage{bibentry}
\begin{document}

\maketitle

\begin{abstract}
Image clustering is a classic problem in computer vision, which categorizes images into different groups. Recent studies utilize nouns as external semantic knowledge to improve clustering performance. However, these methods often overlook the inherent ambiguity of nouns, which can distort semantic representations and degrade clustering quality. To address this issue,  we propose a hierar\textbf{C}hical sem\textbf{A}ntic alignm\textbf{E}nt method for image clustering, 
dubbed \textbf{CAE}, which improves clustering performance in a training-free manner. In our approach, we incorporate two complementary types of textual semantics: caption-level descriptions, which convey fine-grained attributes of image content, and noun-level concepts, which represent high-level object categories. 
We first select relevant nouns from WordNet and descriptions from caption datasets to construct a semantic space aligned with image features. 
Then, we align image features with selected nouns and captions via optimal transport to obtain a more discriminative semantic space.
Finally, we combine the enhanced semantic and image features to perform clustering. Extensive experiments across 8 datasets demonstrate the effectiveness of our method, notably surpassing the state-of-the-art training-free approach with a 4.2\% improvement in accuracy and a 2.9\% improvement in adjusted rand index (ARI) on the ImageNet-1K dataset.
\end{abstract}
\section{Introduction}
\label{sec:intro}
\begin{figure}
  \centering
    \includegraphics[width=1\linewidth]{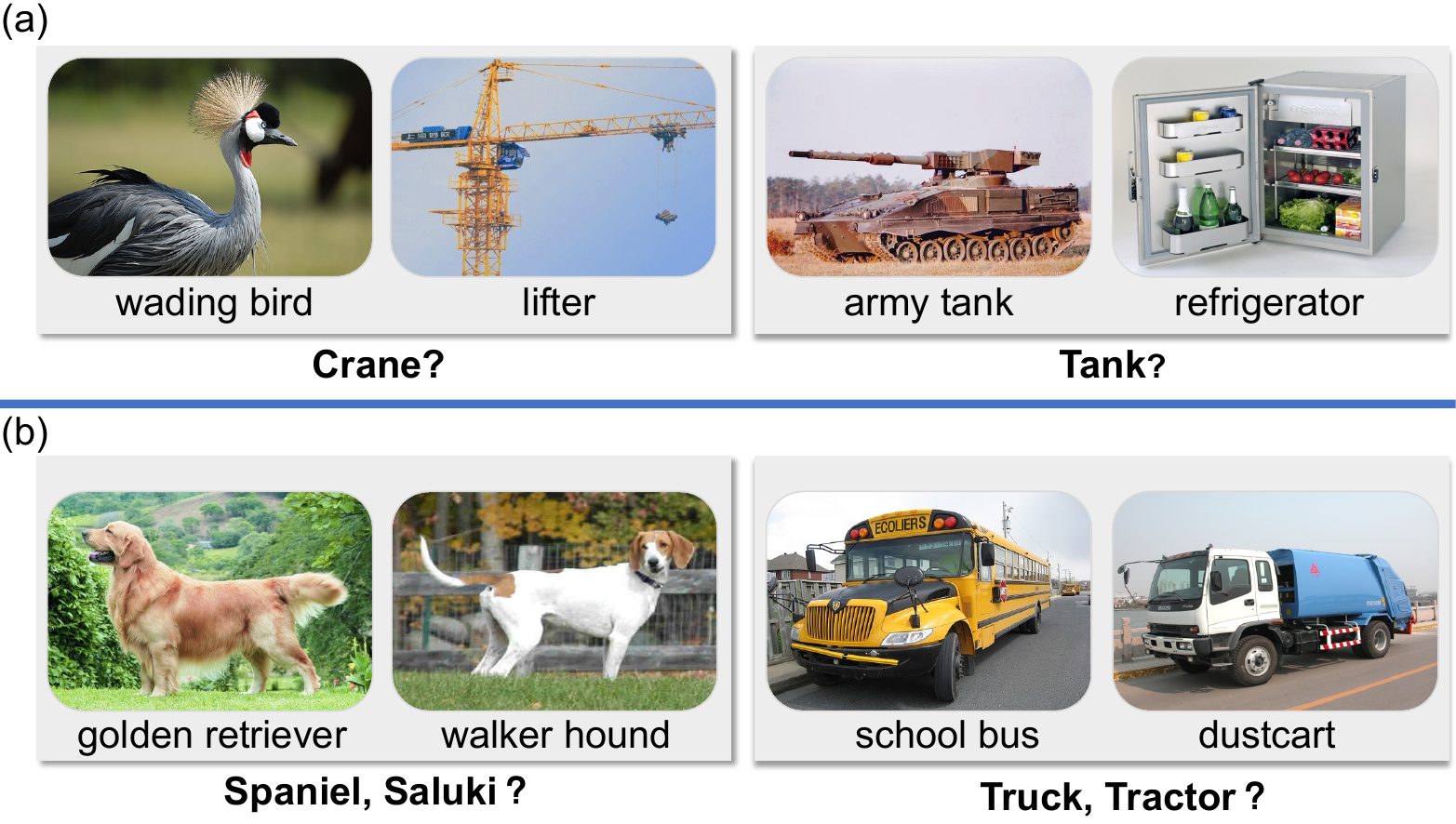}
    \caption{Observations of ambiguity in a single noun from the ImageNet~\cite{ImageNet} dataset. In (a), the words “crane” and ``tank'' can refer to entirely different objects, respectively. In (b), the semantic similar words ``spaniel'', ``saluki'', ``truck'', and ``tractor'' fail to distinguish the fine-grained classes.}
    \label{fig:motivation}
\end{figure}
Image clustering~\cite{likas2003global} is a foundational task in computer vision that aims at grouping images into clusters, where instances from the same cluster share similar semantics.
In the era of deep learning, a series of works center on learning discriminative features~\cite{qiuEOC, protomm} for clustering, such as contrastive learning~\cite{CC, TCL} and self-supervised learning~\cite{MOCO, SimCLR}.
The recent progress of deep clustering has reached a plateau, as the internal priors from the images themselves provide limited improvement. To address this, recent methods seek to incorporate external priors, such as text, to 
 enhance clustering performance.  For instance, Cai~\etal~\cite{SIC} map images into a semantic space constructed from related nouns in WordNet~\cite{WordNet}. Similarly, Li~\etal~\cite{TAC} combine noun embeddings with image embeddings to facilitate clustering. 
 
 Although these methods leverage the rich semantics encoded in textual data, addressing that the inclusion of external knowledge can enhance the ability to differentiate between visually similar clusters, they overlook the inherent ambiguities of nouns. These ambiguities can ultimately diminish the effectiveness of clustering, particularly in cases where the same noun can refer to multiple concepts or subclasses. As shown in Figure~\ref{fig:motivation}~(a),  the noun ``crane'' can refer to a bird or a lifting machine, and the ``tank'' could denote a military vehicle or a refrigerator. 
Additionally, Figure~\ref{fig:motivation} (b) illustrates that a similar noun fails to represent the specific characteristics of fine-grained classes. For example, the term ``spaniel'' or ``saluki'' does not distinguish between a ``golden retriever'' and a ``walker hound,'' nor does the noun ``truck'' or ``tractor'' clearly differentiate between a ``school bus" and a ``dustcart." Therefore,  relying solely on nouns as external knowledge for clustering struggles to provide precise semantics, especially when faced with ambiguities arising from polysemy or multiple subclasses.

\begin{figure}
  \centering
    \includegraphics[width=1\linewidth]{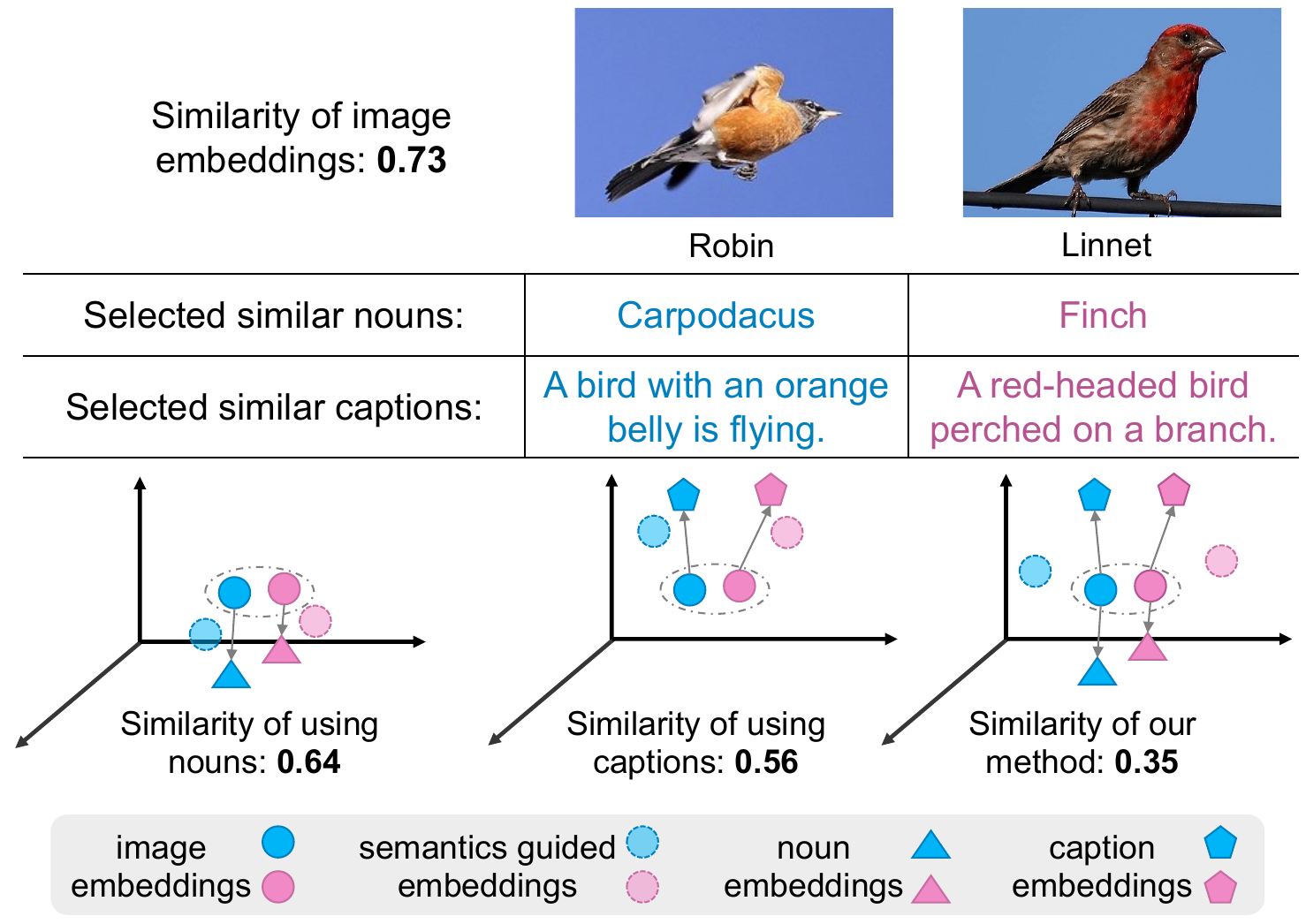}
    \caption{Comparison of embeddings similarity using nouns, captions, and our proposed method for two similar bird images (``Robin" and ``Linnet"). Despite the high similarity in image embeddings (0.73), using only nouns or captions yields higher semantic similarity (0.64 and 0.56, respectively). By combining both nouns and captions, our method reduces the similarity score to 0.35, providing a more accurate distinction between the images.}
    \vspace{-0.3cm}
    \label{fig:case}
\end{figure}
To address these challenges, we propose a training-free method, dubbed \textbf{CAE}, that combines caption-based semantics from image descriptions~\cite{Flick} with noun-based semantics from WordNet~\cite{WordNet} for guided image clustering. These two types of semantics are complementary: nouns capture high-level object categories, while captions provide fine-grained attribute details of the image content.
By combining these two forms of semantics, we can achieve a more comprehensive understanding of the image. As illustrated in Figure~\ref{fig:case}, we compare the effectiveness of using only nouns, only captions, and our proposed combination of both. 
Consider two visually similar bird images from different classes: a robin and a linnet with a cosine similarity of 0.73 between their image embeddings (solid circle). When using only nouns, we select the most similar terms from WordNet: ``Carpodacus" for the robin and ``Finch" for the linnet. These nouns (triangle) are relevant birds but are not the true labels, so they introduce some discrepancy and result in decreasing the similarity from 0.73 to 0.64, as the categorical distinctions between the nouns capture only part of the semantic difference between the images.
When using captions alone, ``A bird with an orange belly is flying" for the robin and ``A red-headed bird perched on a branch" for the linnet. These captions (pentagon) emphasize the differences in the birds' characteristics (orange belly versus red-headed) but retain the common object bird, contributing to reducing the similarity to 0.56.  Although captions offer a contextual distinction, the shared object limits their ability to differentiate the images fully. However, our method combines both nouns and captions, achieves a more distinguishable representation by leveraging both the categorical distinction from nouns and the contextual details from captions, and leads to a significant reduction in similarity, from 0.73 to 0.35.

Specifically, our method consists of two components: semantic space construction and semantic space interaction. In the semantic space construction stage, we first assign all nouns and captions to image semantic centers and then select relevant nouns and captions to construct the counterparts for images. In the adaptive semantics fusion module, we enable image, noun, and caption features to interact through prototype-guided weighting based on their semantic similarity. Finally, the fused representations are used to guide clustering in a more discriminative way.
% In the semantic space interaction part, we allow selected nouns and captions to collaborate with each other to further complement their semantics.
%  Finally, the refined nouns and captions are combined with the images to guide clustering effectively. 

The main contributions are summarized as follows:
\begin{itemize}
    \item We propose a training-free method that exploits the external semantic knowledge from both nouns and captions to effectively guide the clustering process. 
    \item We first construct a semantic space for images by aligning its distribution with relevant nouns and descriptions. Then we design an adaptive fusion strategy by leveraging this space to represent the images.
    % design an attention strategy by leveraging this space to accurately represent images. 
    % % collaborate this space by designing a residual attention mechanism to accurately represent the images.
    \item We demonstrate the effectiveness of our proposed method through extensive experiments on five classic datasets and three challenging datasets, showing consistent and significant improvements over existing baselines, including outperforming zero-shot CLIP.
\end{itemize}

    % \item We develop a strategy to capture the image semantics from nouns and descriptions by distribution distance via optimal transport.
    % \item We design a residual attention mechanism that allows noun and caption embeddings to mutually enhance each other, which enriches the image semantics.

%spaniel
%Saluki

%truck
%tractor

\section{Related Works}
\label{sec:related_work}
% In this section, we first introduce the traditional deep clustering methods followed by the pre-trained vision-language models and their application in cluster tasks.
In this section, we begin with a concise review of deep clustering methods, followed by a discussion on pre-trained vision-language models and their applications.
\begin{figure*}[htbp]
  \centering
    \includegraphics[width=1\linewidth]{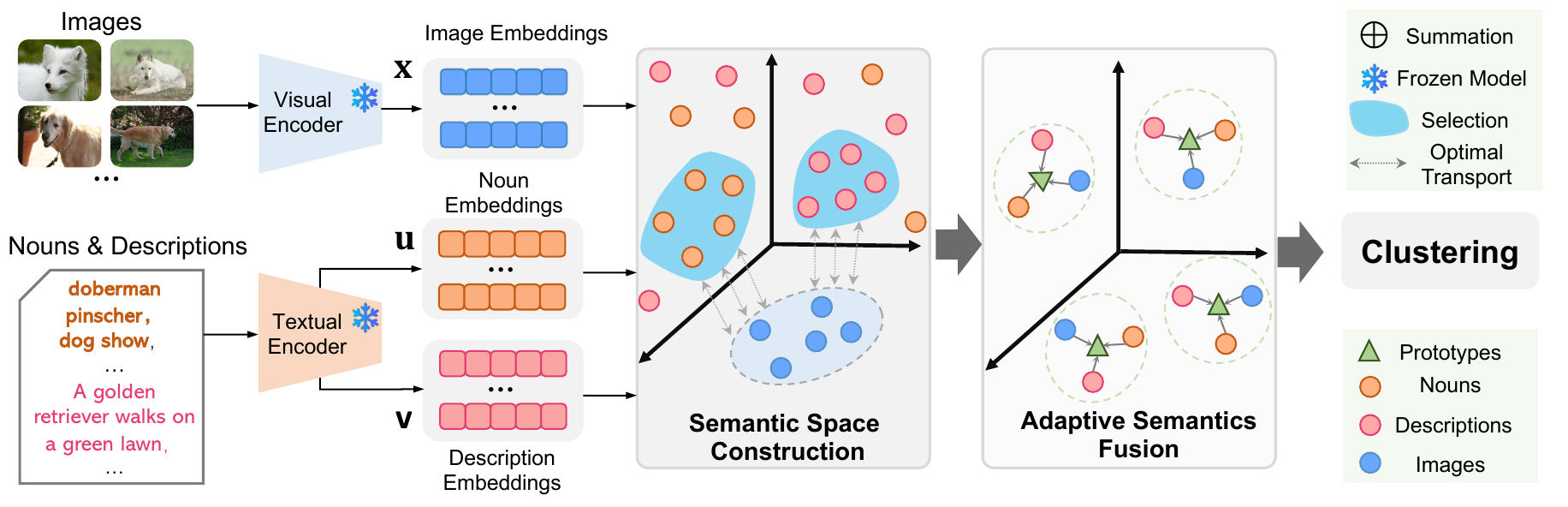}
    \caption{An overview of our method, which consists of two components, \ie, (a) \textbf{Semantic Space Construction:} Select the nouns and descriptions that include the same semantics with image embeddings. (b) \textbf{Adaptive Semantics Fusion:} The selected nouns and descriptions embeddings are leveraged to boost image semantics through adaptive fusion.}
    \label{fig:framework}
\end{figure*}

% \subsection{Deep Image Clustering}
\subsection{Deep Clustering}
To tackle high-dimensional real-world data, deep clustering methods leverage neural networks to learn discriminative features, enabling the effective categorization of samples into distinct clusters~\cite{DeepClusteringPeng2016, JULE, DEC, DeepClusteringPeng2018}. Early deep clustering research focuses on adapting classic clustering objectives~\cite{Yang2017Towards, ji2017deep, shaham2018spectral} into loss functions to optimize neural networks. Recently, motivated by the success of contrastive learning~\cite{SimCLR, MOCO}, contrastive clustering methods leverage the augmentation invariance at instance-~\cite{IMSAT, IIC} or cluster-level~\cite{CC, PICA} to achieve end-to-end clustering. Another branch of research resorts to pseudo-labeling techniques to further boost the clustering performance~\cite{SCAN, TCL, SPICE}. Notably, different from previous multi-modal clustering methods~\cite{CCA, MaoYGY21} that require paired image-text data as input, externally-guided clustering method~\cite{SIC, TAC} does not rely on prior knowledge such as the pairing information and aims at exploring broader external knowledge.

Our work further explores the externally-guided clustering paradigm, advancing both the scope and application of external knowledge. Compared to existing works that leverage semantic information through cross-modal distillation~\cite{TAC} or pseudo-labeling~\cite{SIC}, we demonstrate that image caption models offer promising guidance for clustering. Moreover, We propose a more effective strategy to incorporate image and textual semantics, through the optimal transport and a
residual attention mechanism.

\subsection{Zero/Few-shot Classification}
With advances in pre-training techniques, vision-language pre-training models (VLM), such as CLIP~\cite{CLIP}, have demonstrated a remarkable ability to align images and texts within a unified feature space. Thanks to the incorporation of textural semantics, these models readily adapt to a variety of downstream tasks, such as classification~\cite{MartinHSPA24, APE, zhu2025enhancing}, in few-shot~\cite{TangLW0H24, SSP} or even zero-shot~\cite{CHiLS, Ge0G00AILZ23, VIC, VisDesc, frolic} scenarios. For example, they perform classification by calculating the similarity between the test image and predefined class names.  To better leverage the semantics in class names, Menon~\etal~\cite{VisDesc} proposed prompting a large language model to generate distinct descriptions for specific classes. Novack~\etal~\cite{CHiLS} enriched the class names through a hierarchical fine-grained label set. It is worth noting that Conti~\etal~\cite{VIC} addresses the open vocabulary classification problem by using a multimodal large language model (MVLM) to generate descriptions and then parse class names for each image. However, most of these methods rely heavily on the availability of class names as prior knowledge, which may not always be accessible in real-world scenarios, particularly in clustering tasks. As a result, the potential of the text modality in VLMs cannot be fully utilized, leading to sub-optimal performance.

In this work, we present a more flexible approach to using VLMs by incorporating nouns and descriptions that are not predefined, offering a solution to the challenge of lacking class name prior knowledge. Specifically, we propose an advanced text space construction strategy that effectively captures and leverages visual semantics, thereby enhancing visual representations for image clustering tasks.  
We hope this strategy will inspire broader applications of VLMs in fully unsupervised tasks.

\section{Method}
In this section, we present our training-free method as illustrated in Figure~\ref{fig:framework}. We begin by introducing the settings and notations used throughout our approach in Section~\ref{sec:pre}. Following this, we describe our method in detail, including the construction of the semantic space in Section~\ref{sec:cons} and the semantics fusion in Section~\ref{sec:semantics_fusion}.
% An overview of our method is illustrated in Figure~\ref{fig:framework}.

\subsection{Preliminaries}\label{sec:pre}
{Giving an image dataset $\{x_i\}_{i=1}^N$ with $N$ samples, whose embeddings are computed as $\x_i = \Ev({x_i})$, where $\Ev$ represents CLIP's visual encoder. To capture the semantics of these images, we introduce two textual datasets. The first is a nouns dataset $\{\bu_i\}_{i=1}^{T}$ containing $T$ words from WordNet~\cite{WordNet}, a large English lexical database that groups words into sets of synonyms called synsets. The second is a captions dataset $\{\bv_i\}_{i=1}^M$ with $M$ captions from Flickr~\cite{Flick, FangIWWSDS22}, an online photo-sharing platform where images are often accompanied by user-generated captions. The noun embeddings $\mathbf{u}_i = \Et({u}_i)$ and  caption embeddings $\mathbf{v}_i = \Et({v}_i) $ are calculated using CLIP's textual encoder $\Et$. The image embeddings and textual embeddings share the same dimension, where $\x_i, \mathbf{u}_i, \mathbf{v}_i \in \mathbb{R}^{d}$.  The goal is to assign the images into $K$ clusters.}

\subsection{Semantic Space Construction}\label{sec:cons}
\noindent{\textbf{Filtering nouns and descriptions.}} External semantic knowledge has been shown to enhance the performance of image classification and clustering tasks~\cite{SPICE,CLIP,TAC,SIC,AFR}. In this paper, we leverage textual semantics to assist clustering tasks. Notably, clustering tasks lack priors of object names or object attributes for individual images. To address this, we construct a semantic space by selecting a subset of nouns and captions from WordNet and Flickr that encompass the semantics of objects and attributes, enabling the description of images in the text modality.

A well-constructed semantic space should focus on image-relevant semantics while excluding unrelated ones to prevent confusion in clustering.
A small space risks losing critical discriminative information, whereas a large one may introduce excessive general semantics, diminishing clustering effectiveness.
To ensure an appropriate size, we follow the setting in~\cite{TAC}: the dataset is clustered into $n = N / 300$ groups, and the closest nouns and captions are retrieved for each cluster to define the semantic space.
Specifically, we apply k-means clustering on the image embeddings and compute the semantic center of the $j$-th cluster:
\begin{gather}
    \mathbf{p}_j = \frac{1}{|\mathcal{P}_j|} \sum_{i\in \mathcal{P}_j} \x_i, \ j \in [1, n],
    \label{eq:img_k_means}
\end{gather}
where $\mathcal{P}_j$ denotes the set of images assigned to the $j$-th cluster.
To identify representative nouns, we first define the probability of the $i$-th noun belonging to the $j$-th cluster as:
\begin{equation}
    p(\mathbf{p}_j|\mathbf{u}_i) = \frac{\exp(\mathbf{u}_i^\top\mathbf{p}_j)}{\sum_{l=1}^n \exp(\mathbf{u}_i^\top\mathbf{p}_l)}. \label{eq:noun_retrieval} 
\end{equation}
We reorder the nouns embeddings according to decreasing probability $p(\mathbf{p}_j|\mathbf{u}_i)$, denoting the reordered sequence as $[\bu_{(1)},...\bu_{(i)},...,\bu_{(T)}]$.  
We then retrieve the top-$K$ nouns candidates with the highest probability for the $j$-th cluster:
\begin{equation}
    \mathcal{U}_j=\{\bu_{(i)}|\ i \leq K\}.
    \label{eq:topC}
\end{equation}
Finally, we take the union of these sets to form the selected nouns across all clusters: $\mathcal{U}=\bigcup_{j=1}^n \mathcal{U}_j$. The closest captions, $\mathcal{V}$, for all clusters are in the same manner.

% \textbf{\nointend{imagexxxx}}
% \textbf{\}
% image和Nouns是怎么配对的
%algorithm{}
\noindent{\textbf{Constructing counterparts.}} The selected nouns and captions capture different aspects of image semantics: nouns typically represent general object names, and captions often include fine-grained attributes.
Leveraging these selected embeddings, $\mathcal{U}$ and $\mathcal{V}$, we aim to compute the noun counterpart and the caption counterpart for each image $\x_i$ by assessing the distributional distances between image and textual embeddings. To achieve this, we employ Optimal Transport (OT)~\cite{villani2009optimal}, to aligning their distributions. Mathematically, we define two empirical distributions $P$ and $Q$ to model the sets of two modalities:
\begin{equation}
    {P} = \sum_{i=1}^{N} \frac{1}{N}\delta_{\x_i}, \ {Q} = \sum_{j=1}^{|\mathcal{U}|} \frac{1}{|\mathcal{U}|}\delta_{\bu_j}, 
\label{eq:ot_def}
\end{equation}
where $\delta_{\x_i}$ and $\delta_{\bu_j}$ are the Dirac delta function centered at $\x_i$ and $\bu_j$, respectively. The OT distance between $P$ and $Q$ is thus defined as: 
\begin{equation}
d_{\mathrm{OT}}(P,Q;\mathbf{C}):=\min_{\mathbf{T}\in\Pi(P,Q)}\langle\mathbf{T},\mathbf{C}\rangle, 
\label{eq:ot}
\end{equation} 
where $\langle\cdot,\cdot\rangle$ denotes the Frobenius dot-product, $\textbf{T}\in\mathbb{R}^{N\times|\mathcal{U}|}$ is the transport plan, and $\textbf{C}\in\mathbb{R}^{N\times|\mathcal{U}|}$ the cost matrix. Each element $c_{i,j}$ is calculated using cosine similarity:
\begin{equation}
    c_{i,j} = 1 - s_{i,j}^{u},
    % {\x_{i}\top}{\bu_{j}/(||{\x}_{i}}||\cdot||\bu_{j}||).
\end{equation}
where $s_{i,j}^{u} = {\x_{i}\top}{\bu_{j}/(||{\x}_{i}}||\cdot||\bu_{j}||)$. The goal of OT is to minimize the total cost to transporting mass from $\x$ to $\bu$.
To achieve this, we use the Sinkhorn-Knopp algorithm~\cite{sinkhorn} to approximate the solution $\textbf{T}^{u}$. 
% and details are provided in Appendix~\ref{sec:supp_ot}. 
After obtaining the transport plan $\textbf{T}^{u}$, we compute the noun counterpart $\x^u_i$ for each image $\x_i$:
\begin{equation}
    \x^u_i = \sum_{j=1}^{|\mathcal{U}|}t_{i,j}^us_{i,j}^{u}\bu_j 
    \label{eq:noun_counterpart}
\end{equation}
where $t_{i,j}$ is the $i$-th column and $j$-th row element in $\textbf{T}^{u}$. 
Following the same procedure, we also compute the caption counterpart $\x^v_i$ for each image $\x_i$:
\begin{equation}
    \x^v_i = \sum_{j=1}^{|\mathcal{V}|}t_{i,j}^vs_{i,j}^{v}\bv_j,
    \label{eq:des_counterpart}
\end{equation}
where $t_{i,j}^v$ is the element of optimal transport plan $\textbf{T}^v$ between $\x_i$ and $\bv_j$. While optimal transport offers a principled matching strategy,  an alternative and commonly used approach is to apply similarity-based softmax weighting. Specifically, one may compute \(\x^u_i = \sum_{j=1}^{|\mathcal{U}|} w_{i,j}^u \bu_j\), with softmax weights \(w_{i,j}^u = \frac{\exp(s_{i,j}^u)}{\sum_{k=1}^{|\mathcal{U}|} \exp(s_{i,k}^u)}\). We compare this similarity+softmax approach to OT in the following theorem.
\begin{theorem}
\label{thm:ot_vs_softmax}
Let $\{\x_i\}_{i=1}^N \subset \mathbb{R}^d$ be a set of image embeddings, and let $ \{\bu_j\}_{j=1}^{|\mathcal{U}|} \subset \mathbb{R}^d$ denote a set of textual embeddings (e.g., noun or caption representations). For clarity, we present the theoretical analysis using noun embeddings as a representative case. The softmax-based aggregation:
\begin{equation}
    \x_i^u = \sum_{j=1}^{|\mathcal{U}|} w_{i,j}^u \bu_j, \quad w_{i,j}^u = \frac{\exp(s_{i,j}^u)}{\sum_{k=1}^{|\mathcal{U}|} \exp(s_{i,k}^u)},
\end{equation}
is a special case of entropic OT when the column-wise marginal constraint \(\sum_i t_{i,j}^u = \frac{1}{|\mathcal{U}|}\) is relaxed. For any \(\delta > 0\), with probability at least \(1 - \delta\) over Sinkhorn iterations, softmax incurs higher semantic error than OT.
\end{theorem}

\begin{proof}
The entropic OT formulation defines transport weights as:
\begin{equation}
    t_{i,j}^u = a_i b_j \exp\left(-\frac{c_{i,j}}{\epsilon}\right), \quad c_{i,j} = 1 - s_{i,j}^u,
\end{equation}
where $a_i$, $b_j$ are positive scaling factors. These factors are determined via Sinkhorn iterations such that the transport matrix $\mathbf{T}^u$ satisfies the marginal constraints:
\begin{equation}
    \sum_j t_{i,j}^u = \frac{1}{N}, \quad \sum_i t_{i,j}^u = \frac{1}{|\mathcal{U}|}.
\end{equation}
Softmax weights correspond to:
\begin{equation}
    t_{i,j}^u = \frac{w_{i,j}^u}{N}, \quad w_{i,j}^u = \frac{\exp(s_{i,j}^u)}{\sum_k \exp(s_{i,k}^u)},
\end{equation}
which satisfies \(\sum_j t_{i,j}^u = \frac{1}{N}\), but generally violates the column-wise constraint \(\sum_i t_{i,j}^u = \frac{1}{|\mathcal{U}|}\), i.e., noun embeddings are not used in a balanced way. As for optimization, softmax minimizes a per-image cost \(\sum_j w_{i,j}^u c_{i,j} + \sum_j w_{i,j}^u \log w_{i,j}^u\) subject to \(\sum_j w_{i,j}^u = 1\), while OT globally minimizes a regularized cost under marginal constraints. Hence,
\begin{equation}
    E_{\text{softmax}} = \sum_{i,j} \frac{w_{i,j}^u}{N} c_{i,j} \geq \sum_{i,j} t_{i,j}^u c_{i,j} = E_{\text{OT}},
\end{equation}
as OT minimizes Eq.~\eqref{eq:ot}. With probability at least \(1 - \zeta\) over Sinkhorn iterations~\cite{sinkhorn}, OT’s column constraint, via \(b_j\), balances noun usage, reducing intra-cluster variance \(\mathbb{E}_{(i,i') \in Z_k} \left[ \left\| \sum_j t_{i,j}^u \bu_j - \sum_j t_{i',j}^u \bu_j \right\|^2 \right] \leq \mathbb{E}_{(i,i') \in C_k} \left[ \left\| \sum_j w_{i,j}^u \bu_j - \sum_j w_{i',j}^u \bu_j \right\|^2 \right]\) for cluster \(Z_k\), enhancing clustering. 
% Details are in Appendix~\ref{subsec:cluster_variance}.
\end{proof}

\subsection{Adaptive Semantics Fusion}\label{sec:semantics_fusion}
% Given $\x, \x^u$ and $\x^v$, a straightforward approach to clustering is to apply k-means to the concatenated embeddings $\{[\x_i;\x^u_i; \x^v_i]\}_{i=1}^{N}$, while this strategy achieves improvements as discussed in Section~\ref{subsubsec:combine}, it doesn't consider the collaboration between $x$, $\x^u$ and $\x^v$. To more effectively capture the semantics embedded within visual and textual modalities, we propose an adaptive fusion strategy that facilitates a 
Given multimodality features ${\x}_i$, ${\x}_i^u$, and ${\x}_i^v$ extracted from image, noun, and caption modalities respectively, a naive strategy is to concatenate them into $\{[{\x}_i; {\x}_i^u; {\x}_i^v]\}_{i=1}^N$ and perform k-means clustering. Although effective to some extent (Section~\ref{subsubsec:combine}), this approach neglects the varying contribution and interaction among modalities. To enable adaptive feature fusion, we introduce a prototype-guided weighting mechanism. 
Let ${\x}_i$, ${\x}_i^u$, and ${\x}_i^v$ represent the image, noun, and caption features for the $i$-th instance. We define a unified modality set $\mathcal{M} = \{(m): m \in \{{\x}_i, {\x}_i^u, {\x}_i^v\}\}$ and compute the semantic prototype by averaging over all modalities:
\begin{equation}
    \x_i^{p} = \frac{1}{|\mathcal{M}|} \sum_{m \in \mathcal{M}} {\x}_i^{(m)}.
\end{equation}

Then, we compute the cosine similarity between each modality-specific feature and the prototype:
\begin{equation}\label{eq:simi_alpha}
    \alpha_i^{(m)} = \frac{\langle {\x}_i^{(m)}, {\x}_i^p \rangle}{\|{\x}_i^{(m)}\| \cdot \|{\x}_i^p\|}, \quad \forall m \in \mathcal{M}.
\end{equation}
Let $\bm{\alpha}_i = [\alpha_i^{(m)}]_{m \in \mathcal{M}}$ be the similarity vector. We apply temperature-scaled softmax to obtain the modality weights:
\begin{equation}
    \bm{\beta}_i = \text{softmax}\left( \bm{\alpha}_i / \gamma \right), \quad \text{with } \bm{\beta}_i = [\beta_i^{(m)}]_{m \in \mathcal{M}}.
\end{equation}
Finally, the fused representation for each instance is obtained via a weighted combination of modality-specific features:
\begin{equation}\label{eq:fused_fea}
    \bar{{\x}}_i = \sum_{m \in \mathcal{M}} \beta_i^{(m)} {\x}_i^{(m)}.
\end{equation}
This formulation enables instance-wise adaptive fusion, allowing the model to emphasize modalities that are more semantically aligned with the visual content. Such content-aware integration facilitates fine-grained cross-modal representation learning. Subsequently, we perform clustering on the fused embeddings using the standard k-means algorithm:
\begin{equation}
    {y}_i := \text{k-means}(\bar{{\x}}_i), \quad i \in [1, N],
    \label{eq:results}
\end{equation}
where \({y}_i\) denotes the predicted cluster assignment for the \(i\)-th image. The complete procedure of our methods is summarized in Algorithm~\ref{alg:logic}.

\begin{algorithm}[tb]
\footnotesize
   \caption{Pipeline of \textbf{CAE}}
   \label{alg:logic}
\begin{algorithmic}[1]
   \REQUIRE Image embeddings $\{\x_i\}_{i=1}^{N}$, noun embeddings $\{\bu_i\}_{i=1}^{T}$, description embeddings $\{\bv_i\}_{i=1}^{M}$.   \\
   \textbf{Step1: construct semantic space}
   \STATE Construct relevant nouns set $\mathcal{U}$ and descriptions set $\mathcal{V}$ by Eq.~\eqref{eq:noun_retrieval} and Eq.~\eqref{eq:topC}.
   \STATE Build transport plan $\textbf{T}^u$ and $\textbf{T}^v$ by optimize Eq.~\eqref{eq:ot}.
   \STATE Compute counterparts $\{\x_i^{u}\}_{i=1}^{N}$ and $\{\x_i^{v}\}_{i=1}^{N}$ by Eq.~\eqref{eq:noun_counterpart} and Eq.~\eqref{eq:des_counterpart}, respectively. \\
   \textbf{Step2: adaptive semantics fusion}
    \STATE Compute similarity weights $\bm{\alpha}$ by Eq.~\eqref{eq:simi_alpha}. \\
    \STATE Compute fused features $\bar{\x}_{i}$ by Eq.~\eqref{eq:fused_fea}.
    \ENSURE Compute cluster assignments $\{{y}_i\}_{i=1}^{N}$ by Eq.~\eqref{eq:results}.
\end{algorithmic}
\end{algorithm}

% where each element of $\mathbf{A}$ represents the similarity of a noun embedding and a description embedding corresponding to one image. Based on this, the interacted noun and caption embeddings are expressed as:
% \begin{gather}
% 	\Tilde{\mathbf{U}} = {\rm softmax}(\mathbf{A}/ \gamma)\hat{\mathbf{U}} + \hat{\mathbf{U}}, \label{eq: U_hat}  \\
% 	\Tilde{\mathbf{V}} = {\rm softmax}(\mathbf{A}^\top/ \gamma)\hat{\mathbf{V}} + \hat{\mathbf{V}}. \label{eq: V_hat}
% \end{gather}
% Here $\gamma$ modulates the scaling of attention weights for the noun and description embeddings, respectively, and the attention weights can emphasize similar components in $\hat{\mathbf{U}}$ and $\hat{\mathbf{V}}$ in enhance their representations. For the noun embeddings $\hat{\mathbf{U}}$, which primarily capture the object-level concepts in images, the attention calculations in Eq.~\eqref{eq: U_hat} adaptively adjust their weights based on the fine-grained semantics from $\hat{\mathbf{V}}$. Similarly, the caption embeddings $\hat{\mathbf{V}}$,  which describe the entire image—including both the foreground and other regions, are refined to focus on more general semantics through attentive interaction with $\hat{\mathbf{U}}$.
%  Finally, we perform cluster on images with the interacted textual embeddings through k-means:
% \begin{equation}
% 	\hat{\mathbf{y}} := \text{k-means}(\Tilde{\mathbf{U}} + \Tilde{\mathbf{V}} + \mathbf{X}),
% \end{equation}
% where $\hat{\mathbf{y}}$ is the predicted labels.

\section{Experiments}
In this section, we present the experimental evaluation of our method, including performance comparisons, ablation studies, and clustering visualizations. 
\subsection{Setup}
\noindent{\textbf{Datesets. }}We conduct experiments on five widely used datasets: STL10~\cite{STL}, CIFAR-10~\cite{CIFAR}, CIFAR-20~\cite{CIFAR}, ImageNet-10~\cite{ImageNet_10_Dogs}, and ImageNet-Dogs~\cite{ImageNet_10_Dogs}, and three challenging datasets: DTD~\cite{DTD}, UCF-101~\cite{UCF101}, and ImageNet-1K~\cite{ImageNet}. These datasets cover a diverse range of image categories, allowing us to assess the generalizability of clustering methods across diverse scenarios. 
% The details of these datasets are summarized in Appendix~\ref{supp:data}.

% The details are summarized in Table~\ref{tab: dataset}.

% \ie, STL-10~\cite{STL} and CIFAR-10~\cite{CIFAR}
% contain generic objects and animals, CIFAR-20~\cite{CIFAR} extends CIFAR-10 by grouping images into 20 superclasses, ImageNet-10~\cite{ImageNet_10_Dogs} is a selective subset of the expansive ImageNet~\cite{ImageNet} dataset, focusing on ten distinct categories to facilitate classification or clustering tasks, ImageNet-Dogs~\cite{ImageNet_10_Dogs} is a specialized subset of the ImageNet dataset focusing exclusively on dog breeds. Additionally, we also use three challenging datasets for clustering evaluation: DTD~\cite{DTD}, UCF-101~\cite{UCF101}, and ImageNet-1K~\cite{ImageNet}. DTD~\cite{DTD} (Describable Textures Dataset) is a texture database consisting of images from 47 classes, designed to study texture recognition in the wild. UCF-101~\cite{UCF101} is an action recognition dataset comprising realistic action videos collected from YouTube, with 101 classes. ImageNet-1K~\cite{ImageNet} is a large-scale dataset containing 1,000 classes and over a million images. These challenging datasets allow us to assess the generalizability of clustering methods across diverse scenarios. The details are summarized in Table~\ref{tab: dataset}.

\noindent{\textbf{Evaluation Metrics.}}
We utilize three popular metrics to evaluate the clustering performance, including  Normalized Mutual Information (NMI), Accuracy (ACC), and Adjusted Rand Index (ARI). The higher values of these metrics indicate the better performance.

\noindent{\textbf{Implementation details.}} To ensure a fair comparison with prior work SIC~\cite{SIC} and TAC~\cite{TAC}, we use the CLIP model~\cite{CLIP} with ViT-B/32 as the image encoder and a Transformer as the text encoder. We assemble the nouns from WordNet in the template as “a photo of [CLASS],” and the captions are filtered in the same manner as used in~\cite{FangIWWSDS22}. The temperature parameter $\gamma$ is set as  0.01, for all datasets. All experiments are conducted on a single Nvidia RTX 3090 GPU.
\vspace{-0.2cm}
\begin{table*}[t]
\caption{Clustering performance on 5 widely-used image clustering datasets. The best and second best results are denoted in \textbf{bold} and \underline{underline}, respectively.}
\label{tab:classic_res}
\centering
\tabstyle{4pt}
% \resizebox{\textwidth}{!}{%
\begin{tabular}{cc|ccc|ccc|ccc|ccc|ccc|cc}
\toprule
 \multicolumn{2}{c|}{\multirow{2}{*}{Method}}    & \multicolumn{3}{c|}{STL-10} & \multicolumn{3}{c|}{CIFAR-10} & \multicolumn{3}{c|}{CIFAR-20} & \multicolumn{3}{c|}{ImageNet-10} & \multicolumn{3}{c|}{ImageNet-Dogs} & \multirow{2}{*}{AVG} \\ \cmidrule{3-17}
&          & NMI  & ACC  & ARI  & NMI  & ACC  & ARI  & NMI  & ACC  & ARI  & NMI  & ACC  & ARI  & NMI  & ACC  & ARI  & \\ \midrule

% \multicolumn{1}{c|}{\multirow{19}{*}{\rotatebox{90}{\scriptsize{Training-based}}}} & JULE (CVPR16)     & 18.2 & 27.7 & 16.4 & 19.2 & 27.2 & 13.8 & 10.3 & 13.7 & 3.3  & 17.5 & 30.0 & 13.8 & 5.4  & 13.8 & 2.8  & 15.5 \\ 
% \multicolumn{1}{c|}{} & DEC (ICML16)            & 27.6 & 35.9 & 18.6 & 25.7 & 30.1 & 16.1 & 13.6 & 18.5 & 5.0  & 28.2 & 38.1 & 20.3 & 12.2  & 19.5 & 7.9  & 21.2 \\
% \multicolumn{1}{c|}{} & DAC (ICCV17)            & 36.6 & 47.0 & 25.7 & 39.6 & 52.2 & 30.6 & 18.5 & 23.8 & 8.8  & 39.4 & 52.7 & 30.2 & 21.9 & 27.5 & 11.1 & 31.0 \\
% \multicolumn{1}{c|}{} & DCCM (ICCV19)           & 37.6 & 48.2 & 26.2 & 49.6 & 62.3 & 40.8 & 28.5 & 32.7 & 17.3 & 60.8 & 71.0 & 55.5 & 32.1 & 38.3 & 18.2 & 41.3 \\
% \multicolumn{1}{c|}{} & IIC (ICCV19)            & 49.6 & 59.6 & 39.7 & 51.3 & 61.7 & 41.1 & 22.5 & 25.7 & 11.7 & --   & --   & --   & --   & --   & --   & -- \\
\multicolumn{1}{c|}{\multirow{19}{*}{\rotatebox{90}{\scriptsize{Training-based}}}} & PICA (CVPR20)            & 61.1 & 71.3 & 53.1 & 59.1 & 69.6 & 51.2 & 31.0 & 33.7 & 17.1 & 80.2 & 87.0 & 76.1 & 35.2 & 35.3 & 20.1 & 52.1 \\
\multicolumn{1}{c|}{} & IDFD (ICLR20)          & 64.3 & 75.6 & 57.5 & 71.1 & 81.5 & 66.3 & 42.6 & 42.5 & 26.4 & 89.8 & 95.4 & 90.1 & 54.6 & 59.1 & 41.3 & 63.9 \\
\multicolumn{1}{c|}{} & SCAN (ECCV20)           & 69.8 & 80.9 & 64.6 & 79.7 & 88.3 & 77.2 & 48.6 & 50.7 & 33.3 & --   & --   & --   & 61.2 & 59.3 & 45.7 & -- \\
\multicolumn{1}{c|}{} & MiCE (ICLR20)           & 63.5 & 75.2 & 57.5 & 73.7 & 83.5 & 69.8 & 43.6 & 44.0 & 28.0 & --   & --   & --   & 42.3 & 43.9 & 28.6 & -- \\
\multicolumn{1}{c|}{} & CC (AAAI21)              & 76.4 & 85.0 & 72.6 & 70.5 & 79.0 & 63.7 & 43.1 & 42.9 & 26.6 & 85.9 & 89.3 & 82.2 & 44.5 & 42.9 & 27.4 & 62.1 \\
\multicolumn{1}{c|}{} & GCC (ICCV21)            & 68.4 & 78.8 & 63.1 & 76.4 & 85.6 & 72.8 & 47.2 & 47.2 & 30.5 & 84.2 & 90.1 & 82.2 & 49.0 & 52.6 & 36.2 & 64.3 \\
\multicolumn{1}{c|}{} & NNM (CVPR21)            & 66.3 & 76.8 & 59.6 & 73.7 & 83.7 & 69.4 & 48.0 & 45.9 & 30.2 & --   & --   & --   & 60.4 & 58.6 & 44.9 & -- \\
\multicolumn{1}{c|}{} & TCC (NeurIPS21)             & 73.2 & 81.4 & 68.9 & 79.0 & 90.6 & 73.3 & 47.9 & 49.1 & 31.2 & 84.8 & 89.7 & 82.5 & 55.4 & 59.5 & 41.7 & 67.2 \\ 
\multicolumn{1}{c|}{} & TCL (IJCV22)             & 79.9 & 86.8 & 75.7 & 81.9 & 88.7 & 78.0 & 52.9 & 53.1 & 35.7 & 87.5 & 89.5 & 83.7 & 62.3 & 64.4 & 51.6 & 71.4 \\
\multicolumn{1}{c|}{} & SPICE (TIP22)         & 81.7 & 90.8 & 81.2 & 73.4 & 83.8 & 70.5 & 44.8 & 46.8 & 29.4 & 82.8 & 92.1 & 83.6 & 57.2 & 64.6 & 47.9 & 68.7 \\ 
\multicolumn{1}{c|}{} & SIC (AAAI23)            & \underline{95.3} & \underline{98.1} & \underline{95.9} & \textbf{84.7} & \textbf{92.6} & \textbf{84.4} & 59.3 & \underline{58.3} & \underline{43.9} & 97.0 & 98.2 & 96.1 & 69.0 & 69.7 & 55.8 & \underline{79.9} \\
\multicolumn{1}{c|}{} & SeCu (ICCV23)             & 70.7 & 81.4 & 65.7 & 79.9 & 88.5 & 78.2 & 51.6 & 51.6 & 36.0 & - & - & - & - & - & - & - \\
\multicolumn{1}{c|}{} & DivClust (CVPR23)           & - & - & - & 71.0 & 81.5 & 67.5 & 44.0 & 43.7 & 28.3 & 85.0 & 90.0 & 81.9 & 51.6 & 52.9 & 37.6 & - \\
\multicolumn{1}{c|}{} & RPSC (AAAI24)             & {83.8}  &92.0  &83.4   &75.4  &85.7  &73.1  &47.6 & 51.8 & 34.1 & 83.0 & 92.7 & 85.8 & 55.2 & 64.0 & 46.5 & 70.2   \\ 
\multicolumn{1}{c|}{} & LFSS (ICML25)      & 77.1 & 86.1 & 74.0 & 87.2 & 93.4 & 86.8 & 59.9 & 58.7 & 43.5 & 85.6 & 93.2 & 85.7 & 61.7 & 69.1 & 53.3 & 74.3 \\ \midrule
 \multicolumn{1}{c|}{\multirow{4}{*}{\rotatebox{90}{\scriptsize{Training-free}}}} & CLIP (k-means)    & 91.7 & 94.3 & 89.1 & 70.3 & 74.2 & 61.6 & 49.9 & 45.5 & 28.3 & 96.9 & 98.2 & 96.1 & 39.8 & 38.1 & 20.1 & 66.3 \\
\multicolumn{1}{c|}{} & TAC (ICML24)      & 92.3 & 94.5 & 89.5 & 80.8 & 90.1 & 79.8 & \underline{60.7} & 55.8 & 42.7 & \underline{97.5} & \underline{98.6} & \underline{97.0} & \underline{75.1} & \underline{75.1} & \underline{63.6} & 79.5 \\
\multicolumn{1}{c|}{} & \textbf{\ours} (Ours)           & \textbf{95.5} & \textbf{98.2} & \textbf{96.7} & \underline{81.7} & \underline{90.9} & \underline{80.5} & \textbf{62.8} & \textbf{60.9} & \textbf{45.9} & \textbf{98.1} & \textbf{98.8} & \textbf{97.4} & \textbf{77.9} & \textbf{77.5} & \textbf{66.3} & \textbf{81.9} \\
\multicolumn{1}{c|}{} & \textcolor{gray}{CLIP (zero-shot)}  & \textcolor{gray}{93.9} & \textcolor{gray}{97.1} & \textcolor{gray}{93.7} & \textcolor{gray}{80.7} & \textcolor{gray}{90.0} & \textcolor{gray}{79.3} & \textcolor{gray}{55.3} & \textcolor{gray}{58.3} & \textcolor{gray}{39.8} & \textcolor{gray}{95.8} & \textcolor{gray}{97.6} & \textcolor{gray}{94.9} & \textcolor{gray}{73.5} & \textcolor{gray}{72.8} & \textcolor{gray}{58.2} & \textcolor{gray}{78.7} \\
\bottomrule
\end{tabular}%
% }
\end{table*}

\begin{table*}[t]
\centering
\caption{Clustering performance on 3 challenging image clustering datasets. The best and second best results are denoted in \textbf{bold} and \underline{underline}, respectively.}
\label{tab: other_complex}
\resizebox{0.8\textwidth}{!}{
\begin{tabular}{cc|ccc|ccc|ccc|c}
\toprule
 \multicolumn{2}{c|}{\multirow{2}{*}{Method}}  & \multicolumn{3}{c|}{DTD} & \multicolumn{3}{c|}{UCF-101} & \multicolumn{3}{c|}{ImageNet-1K} & \multirow{2}{*}{AVG} \\ \cmidrule{3-11}
&  & NMI  & ACC  & ARI  & NMI  & ACC  & ARI  & NMI  & ACC  & ARI  & \\ \midrule

\multicolumn{1}{c|}{\multirow{2}{*}{\small{Training-based}}}  & SCAN (ECCV20)    & 59.4 & \underline{46.4} & \underline{31.7} & 79.7 & 61.1 & 53.1 & 74.7 & 44.7 & 32.4 & 53.7 \\
\multicolumn{1}{c|}{}& SIC (AAAI23)      & 59.6 & 45.9 & 30.5 & 81.0 & \underline{61.9} & \underline{53.6} & 77.2 & 47.0 & 34.3 & 54.6 \\ \midrule
\multicolumn{1}{c|}{\multirow{4}{*}{{\small{Training-free}}}}  & CLIP (k-means)   & 57.3 & 42.6 & 27.4 & 79.5 & 58.2 & 47.6 & 72.3 & 38.9 & 27.1 & 50.1 \\
\multicolumn{1}{c|}{} & TAC (ICML24)     & \underline{60.1} & 45.9 & 29.0 & \underline{81.6} & 61.3 & 52.4 & \underline{77.8} & \underline{48.9} & \underline{36.4} & \underline{54.8} \\
\multicolumn{1}{c|}{} & \textbf{\ours}  (Ours)          & \textbf{62.1} & \textbf{46.7} & \textbf{31.9} & \textbf{82.8} & \textbf{63.5} &      \textbf{54.9} & \textbf{79.3} & \textbf{53.1} & \textbf{39.3} & \textbf{57.0} \\
\multicolumn{1}{c|}{} & \textcolor{gray}{CLIP (zero-shot)}   & \textcolor{gray}{56.5} & \textcolor{gray}{43.1} & \textcolor{gray}{26.9} & \textcolor{gray}{79.9} & \textcolor{gray}{63.4} & \textcolor{gray}{50.2} & \textcolor{gray}{81.0} & \textcolor{gray}{63.6} & \textcolor{gray}{45.4} & \textcolor{gray}{56.7} \\ \bottomrule
\end{tabular}
}
\end{table*}

\subsection{Main Results}
We compare our approach with state-of-the-art (SOTA) methods and categorize the evaluated methods into training-based and training-free. We first evaluate our method on 5 classic clustering datasets, comparing it with 19 training-based methods and 2 training-free methods. Previous works have utilized various backbones, such as ResNet-34 (SPICE~\cite{SPICE}, TCL~\cite{TCL}) and ResNet-18 (SPICE~\cite{SPICE}, SCAN~\cite{SCAN}). 
With the rapid development of model pre-training, CLIP has been adopted in clustering tasks~\cite{TAC, SIC}. The comparison includes a wide range of representative methods, such as 
% DAC~\cite{DAC},
% DCCM~\cite{DCCM},
% ICC~\cite{IIC},
PICA~\cite{PICA}, 
IDFD~\cite{IDFD}, 
Mice~\cite{MiCE}, 
CC~\cite{CC}, 
GCC~\cite{GCC}, 
NNM~\cite{NNM}, 
TCC~\cite{DCCM}, 
SeCu~\cite{SeCu}, 
DivClust~\cite{DivClust}, 
PRSC~\cite{RPSC}, LFSS~\cite{LFSS}. 
% With the rapid development of model pre-training, CLIP has been increasingly adopted in clustering tasks~\cite{TAC, SIC}. The comparison includes a wide range of representative methods, such as PICA~\cite{PICA}, IDFD~\cite{IDFD}, MiCE~\cite{MiCE}, CC~\cite{CC}, GCC~\cite{GCC}, NNM~\cite{NNM}, TCC~\cite{DCCM}, SeCu~\cite{SeCu}, DivClust~\cite{DivClust}, PRSC~\cite{RPSC}, and LFSS~\cite{LFSS}.
Our work mainly centers on comparison with zero-shot CLIP and CLIP-based methods. The results in Table~\ref{tab:classic_res} clearly show that our method consistently outperforms TAC~\cite{TAC} (training-free) on 5 classic datasets, achieving a notable 7.2\% improvement in ARI on STL-10 and a 5.1\% accuracy improvement on CIFAR-20. This successful performance is attributed to the leveraged complementary semantics from selected nouns and captions, which demonstrate the effectiveness of our method. On the other hand, we observe that on the CIFAR-10 dataset, SIC~\cite{SIC} achieves the best performance, supported by the adaptation of features extracted from CLIP, which requires trainable parameters and additional training time. In contrast, our method is training-free and achieves the second-best results. Table~\ref{tab: other_complex} depicts the results on three challenge datasets, and our method still achieves the best performance. Specifically, our \textbf{$\ours$} outperforms TAC~\cite{TAC} over 4\% ACC on Imagenet-1K, which is a significant improvement. Moreover, the last row in Table~\ref{tab: other_complex} indicates the zero-shot performance of CLIP, which relies on the candidate class names of the images. Our method
which depends solely on selected nouns and captions, outperforms zero-shot CLIP on DTD and UCF-101 datasets, highlighting the effectiveness of our approach in applying CLIP for clustering tasks. 
% More experimental results are summarized in Appendix~\ref{supp:challenge_res}--\ref{sec:supp_vis}.
% \vspace{-0.2cm}
% \begin{table}[h]
% \centering
% \caption{Clustering performance of our method using different textual semantics.}
% \label{tab: semantics}
% \resizebox{0.45\textwidth}{!}{%
% \begin{tabular}{lcccc|ccc|ccc}
% \toprule
% &
% \Large{\multirow{2}{*}{${\mathbf{x}}^u$}} & 
% \Large{\multirow{2}{*}{${\mathbf{x}}^v$}} & 
% \Large{\multirow{2}{*}{${\mathbf{\bar{x}}^u}$}} & \Large{\multirow{2}{*}{$\mathbf{\bar{x}}^v$}} & \multicolumn{3}{c|}{ImageNet-Dogs} & \multicolumn{3}{c}{UCF-101} \\ \cmidrule{6-11} 
%  & & &   &   & NMI & ACC & ARI & NMI & ACC & ARI \\ \midrule
% (1)&\ding{52} &   & &        & 75.3 & 75.2 & 63.9& 81.7 & 61.5 & 52.7 \\
% (2)&  &  \ding{52} &   &     & 75.7 &75.5 & 64.0 & 81.4 & 61.3 & 52.3 \\
% (3)&\ding{52}&   \ding{52}&  & &76.4 & 76.3 & 64.5 & 82.3 & 62.1 & 52.9 \\ \midrule
% (4)& &  &  \ding{52}&        & 75.8 & 76.9 & 65.2 & 82.2 & 62.8 & 53.9 \\
% (5)& &   & & \ding{52}& 76.3 & 76.6 & 64.6 & 82.1 & 62.9 & 53.6 \\
% (6)& &  &  \ding{52}& \ding{52}& \textbf{77.9} & \textbf{77.5} & \textbf{66.3} & \textbf{82.8} & \textbf{63.5} & \textbf{54.9} \\ \bottomrule
% \end{tabular}%
% }
% \vspace{-0.2cm}
% \end{table}
% % \vspace{-0.5cm}

\begin{table}[h]
\centering
\caption{Clustering performance of our method using different textual semantics.}
\label{tab: semantics}
\resizebox{0.45\textwidth}{!}{%
\begin{tabular}{lccc|ccc|ccc}
\toprule
&
\Large{\multirow{2}{*}{${\mathbf{x}}^u$}} & 
\Large{\multirow{2}{*}{${\mathbf{x}}^v$}} & 
\Large{\multirow{2}{*}{${\mathbf{\bar{x}}}$}} & 
\multicolumn{3}{c|}{ImageNet-Dogs} & \multicolumn{3}{c}{UCF-101} \\ \cmidrule{5-10} 
 & & &   &    NMI & ACC & ARI & NMI & ACC & ARI \\ \midrule
(1)&\ding{52} &   & &         75.3 & 75.2 & 63.9& 81.7 & 61.5 & 52.7 \\
(2)&  &  \ding{52} &   &      75.7 &75.5 & 64.0 & 81.4 & 61.3 & 52.3 \\
(3)&\ding{52}&   \ding{52}&   &76.4 & 76.3 & 64.5 & 82.3 & 62.1 & 52.9 \\ \midrule
% (4)& &  &  \ding{52}&         75.8 & 76.9 & 65.2 & 82.2 & 62.8 & 53.9 \\
% (5)& &   & & 76.3 & 76.6 & 64.6 & 82.1 & 62.9 & 53.6 \\
(4)& &  &  \ding{52}&  \textbf{77.9} & \textbf{77.5} & \textbf{66.3} & \textbf{82.8} & \textbf{63.5} & \textbf{54.9} \\ \bottomrule
\end{tabular}%
}
\vspace{-0.2cm}
\end{table}

% \vspace{-0.1cm}
\subsection{Ablation Study}
We conduct experiments to assess the effectiveness of leveraged semantics and combination strategies. Additionally, we also examine the impact of the number of image semantic centers, and the number of top-$K$ selections.

% \begin{table*}[!htbp]
% \centering
% \caption{Clustering performance of our method with different combination strategies.}
% \label{tab:combination}
% % \tabstyle{0.9pt}
% \resizebox{0.75\textwidth}{!}{
% \begin{tabular}{l|ccc|ccc|ccc|c}
% \toprule
%  \multicolumn{1}{c|}{\multirow{2}{*}{Method}} & \multicolumn{3}{c|}{DTD} & \multicolumn{3}{c|}{UCF-101} & \multicolumn{3}{c|}{ImageNet-1K} & \multirow{2}{*}{AVG} \\ \cmidrule(r){2-10}
%         & NMI  & ACC  & ARI  & NMI  & ACC  & ARI  & NMI  & ACC  & ARI  & \\ \midrule
% (1) Concat\large{($\mathbf{x}$, ${\mathbf{x}^u}$, ${\mathbf{x}}^v$)}         & 58.4 & 45.4 & 28.5 & 79.7 &59.4  &50.2  &77.9  &49.2  &34.6  & 53.6 \\
% (2) Sum\large{($\mathbf{x}$, ${\mathbf{x}^u}$, ${\mathbf{x}^v}$)}            & 60.1 & 45.8 & 29.8 & 82.3 &62.1  &52.9  &78.7  &50.8  &37.6  & 55.6 \\ \midrule
% % (3) Concat\large{($\mathbf{x}$, $\bar{\mathbf{x}}^u$, $\bar{\mathbf{x}^v}$)} & 60.6 & 46.1 & 28.9 & 81.8 &62.5  &54.7  & 78.8 &49.3  &37.9  & 55.6 \\
% (3) $\sum_{m \in \mathcal{M}} \beta_i^{(m)} {\x}_i^{(m)}$   & \textbf{62.1} & \textbf{46.7} & \textbf{31.9} & \textbf{82.8} & \textbf{63.5} &      \textbf{54.9} & \textbf{79.3} & \textbf{53.1} & \textbf{39.3} & \textbf{57.0} \\ \bottomrule
% \end{tabular}
% }
% \end{table*}
% % (3) Sum\large{($\mathbf{x}$, $\bar{\mathbf{x}}^u$, $\bar{\mathbf{x}}^v$)} 

\begin{table}[!htbp]
\centering
\caption{Clustering performance of our method with different combination strategies.}
\label{tab:combination}
\tabstyle{2pt}
% \resizebox{0.75\textwidth}{!}{
\begin{tabular}{l|ccc|ccc|c}
\toprule
 \multicolumn{1}{c|}{\multirow{2}{*}{Method}} & \multicolumn{3}{c|}{ImageNet-Dogs} & \multicolumn{3}{c|}{UCF-101} &  \multirow{2}{*}{AVG} \\ \cmidrule(r){2-7}
        & NMI  & ACC  & ARI  & NMI  & ACC  & ARI  & \\ \midrule
(1) Concat\large{($\mathbf{x}$, ${\mathbf{x}^u}$, ${\mathbf{x}}^v$)}         & 75.7 & 75.6 & 64.2 & 79.7 &59.4  &50.2  & 53.6 \\
(2) Sum\large{($\mathbf{x}$, ${\mathbf{x}^u}$, ${\mathbf{x}^v}$)}            & 76.3 & 76.2 & 65.6 & 82.3 &62.1  &52.9  & 55.6 \\ \midrule
% (3) Concat\large{($\mathbf{x}$, $\bar{\mathbf{x}}^u$, $\bar{\mathbf{x}^v}$)} & 60.6 & 46.1 & 28.9 & 81.8 &62.5  &54.7  & 78.8 &49.3  &37.9  & 55.6 \\
(3) $\sum_{m \in \mathcal{M}} \beta_i^{(m)} {\x}_i^{(m)}$   & \textbf{77.9} & \textbf{77.5} & \textbf{66.3} & \textbf{82.8} & \textbf{63.5} &      \textbf{54.9}  & \textbf{57.0} \\ \bottomrule
\end{tabular}
% }
\end{table}

\noindent{\textbf{Effectiveness of leveraged semantics.}}
Table~\ref{tab: semantics} presents the results of using different textual semantics with images. In Row (1), when only using ${\mathbf{x}^u}$, it achieves comparable performance with that of Row (2), which uses only ${\mathbf{x}^v}$. When both $\mathbf{\bar{x}}^u$ and $\mathbf{\bar{x}}^v$ are combined in Row (3), there is a noticeable performance improvement, indicating that combining both semantic cues can better capture the image semantics. 
% Rows (4) and (5) show further improvements when using either $\mathbf{\bar{x}}^u$ or $\mathbf{\bar{x}}^v$, surpassing the results of Rows (1) and (2).  This demonstrates the effectiveness of incorporating richer semantic information through our semantic space interaction module. 
% Specifically, using $\Tilde{\mathbf{u}}$ in Row (4) leads to an increase in NMI, ACC, and ARI, showing a more nuanced understanding of the data. 
Finally, Row (4), which employs fused feature $\bar{\x}$, achieves the best results across all metrics and datasets. For instance, on the ImageNet-Dogs dataset, our method achieves a 2.6\% improvement in NMI, a 2.3\% improvement in ACC, and a 2.4\% increase in ARI compared to Row (1). 
% Similar improvements are observed for the UCF-101 dataset, with a notable increase of 1.1\% in NMI, 2.0\% in ACC, and 2.2\% in ARI. These results highlight the significant impact of leveraging complementary noun and caption semantics on improving clustering performance.

% \vspace{-0.45cm}
% \noindent{\textbf{Effectiveness of combination strategies.}} 
% \vspace{-0.25cm}
\noindent{\textbf{Effectiveness of combination strategies.}}\label{subsubsec:combine}
% In this experiment, we validate the effectiveness of different combination strategies between textual semantics with images. 
As discussed in Section~\ref{sec:semantics_fusion}, TAC directly concatenates noun embeddings with image embeddings. In contrast, our method adds both noun embeddings and captions embeddings to the image embeddings. 
The results in Table~\ref{tab:combination}
 demonstrate that the summation operation is more effective for image clustering compared to concatenation. Specifically, when comparing Rows (3), the adaptive fusion operation results in a 3.9\% ACC improvement on the ImageNet-1K dataset using $\mathbf{\bar{x}}^u$ and $\mathbf{\bar{x}}^v$. These improvements suggest that our method better utilizes the semantics of both image and text, thereby enhancing the representativeness and discriminability of the embeddings.

\begin{figure}[t]
\centering
  \subfloat[UCF-101.]{
    \includegraphics[width=0.45\linewidth]{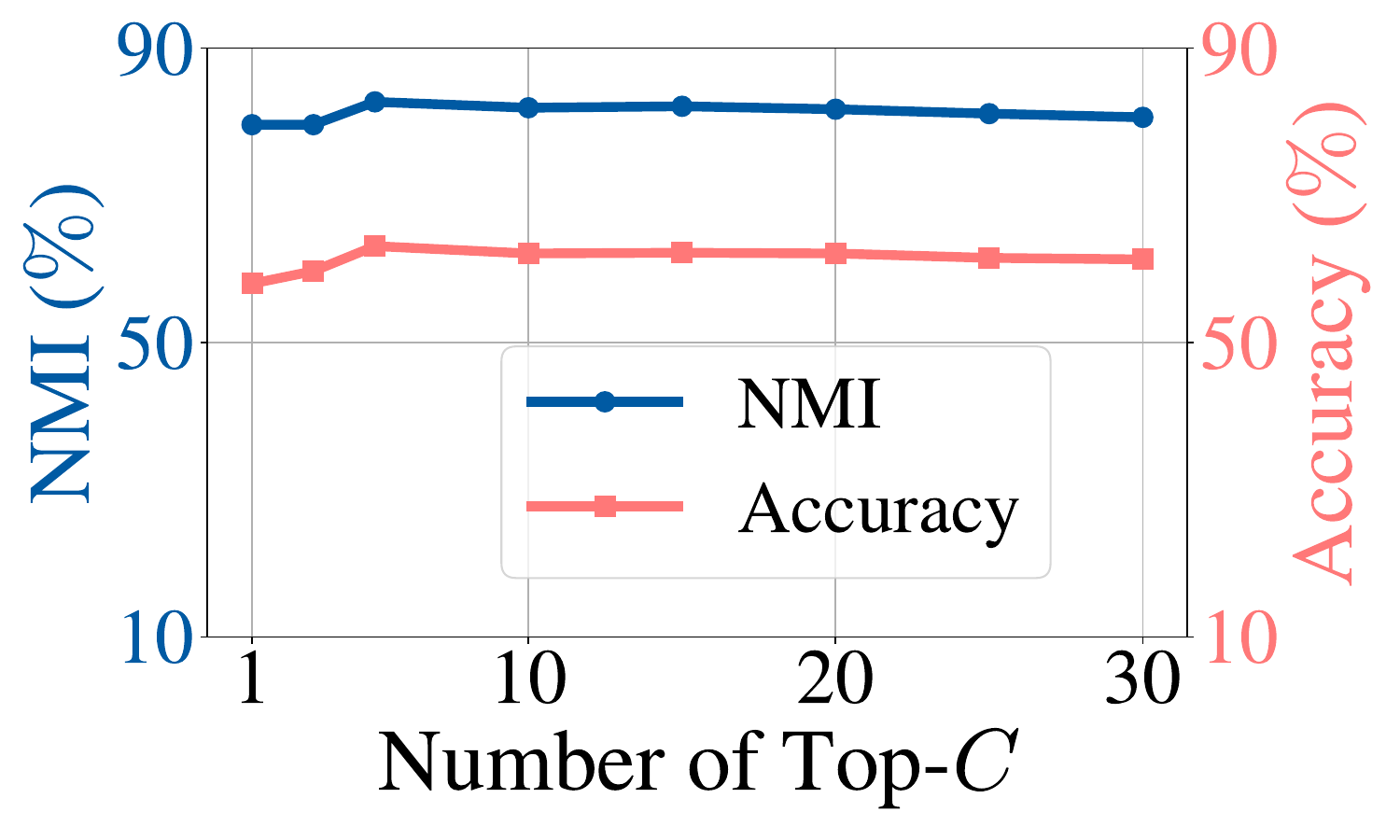}\label{fig:topc_ucf_101}}
  \subfloat[ImageNet-Dogs.]{
    \includegraphics[width=0.45\linewidth]{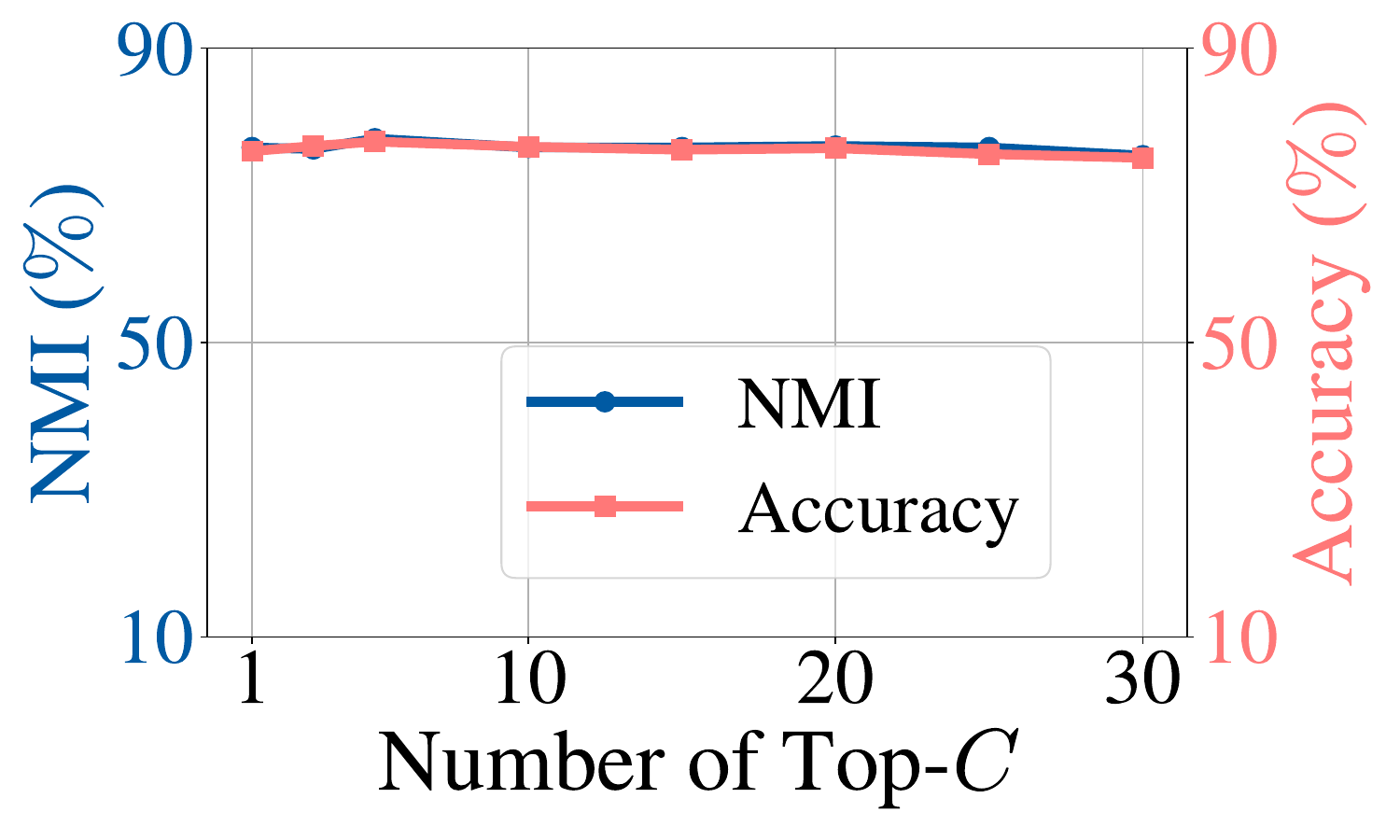}\label{fig:topc_imagenet_dogs}}
  \caption{Analysis of clustering performance by varying the number of image semantic centers on (a) UCF-101 and (b) ImageNet-Dogs datasets, respectively.}\label{fig:img_topc}
  \vspace{-0.5cm}
\end{figure}

\begin{figure}[t]
\centering
  \subfloat[CLIP Image Embeddings.]{
    \includegraphics[width=0.45\linewidth]{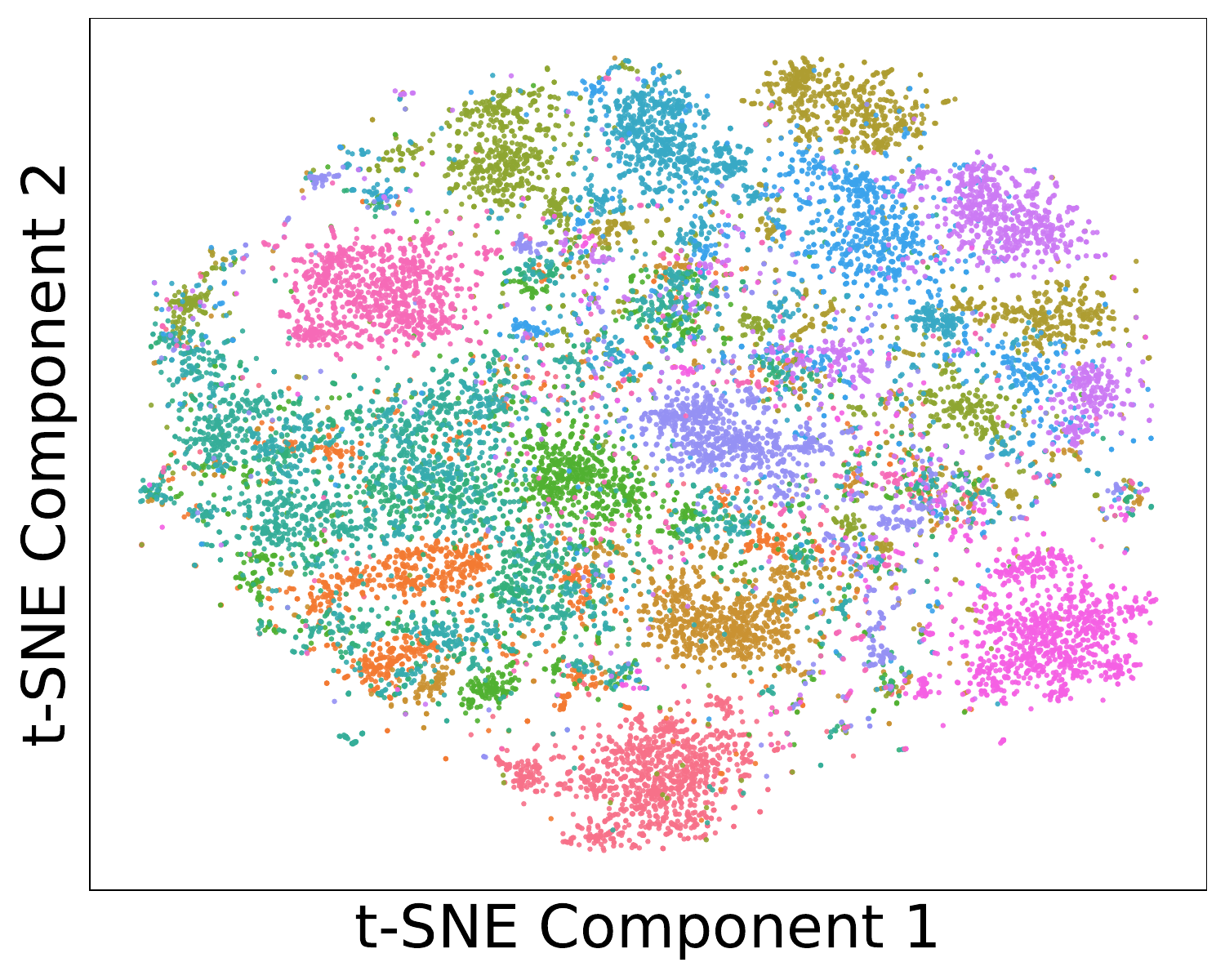}\label{fig:vis_clip}}
  \subfloat[Nouns Embeddings $({\mathbf{x}}^u)$.]{
    \includegraphics[width=0.45\linewidth]{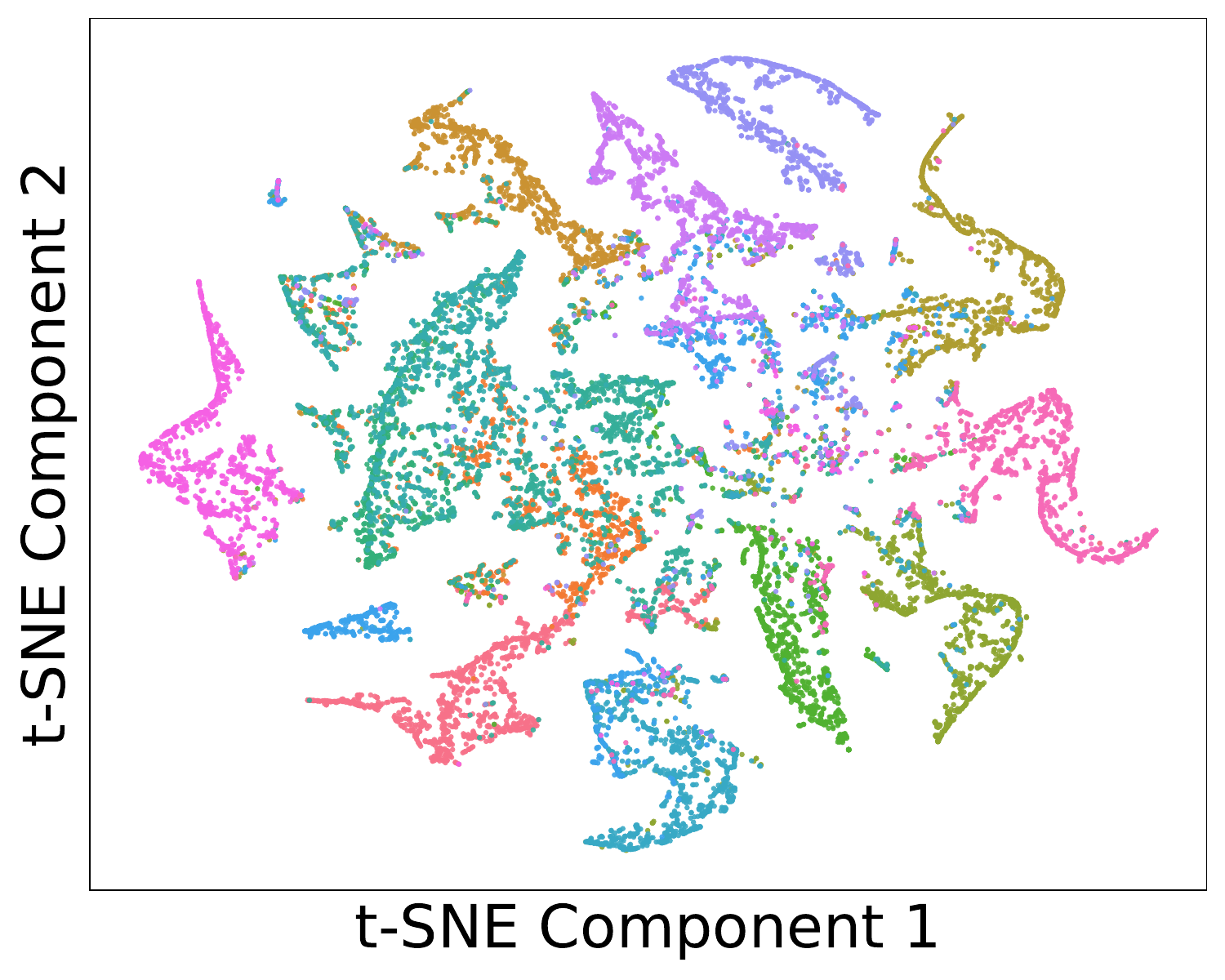}\label{fig:vis_u}} \\
  \subfloat[Caption Embeddings $({\mathbf{x}}^v)$.]{
    \includegraphics[width=0.45\linewidth]{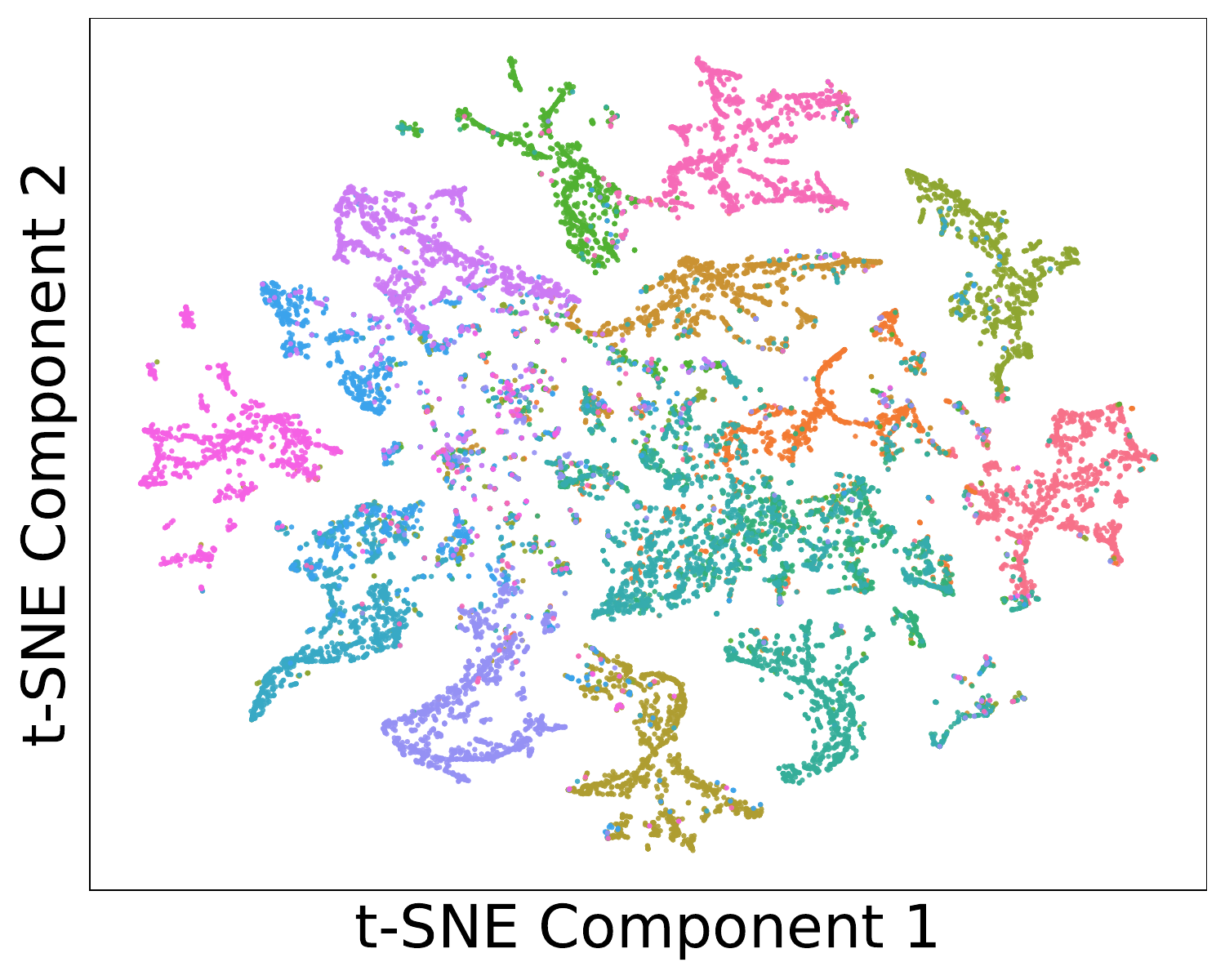}\label{fig:vis_v}}
  \subfloat[Our method (\textbf{\ours}).]{
    \includegraphics[width=0.45\linewidth]{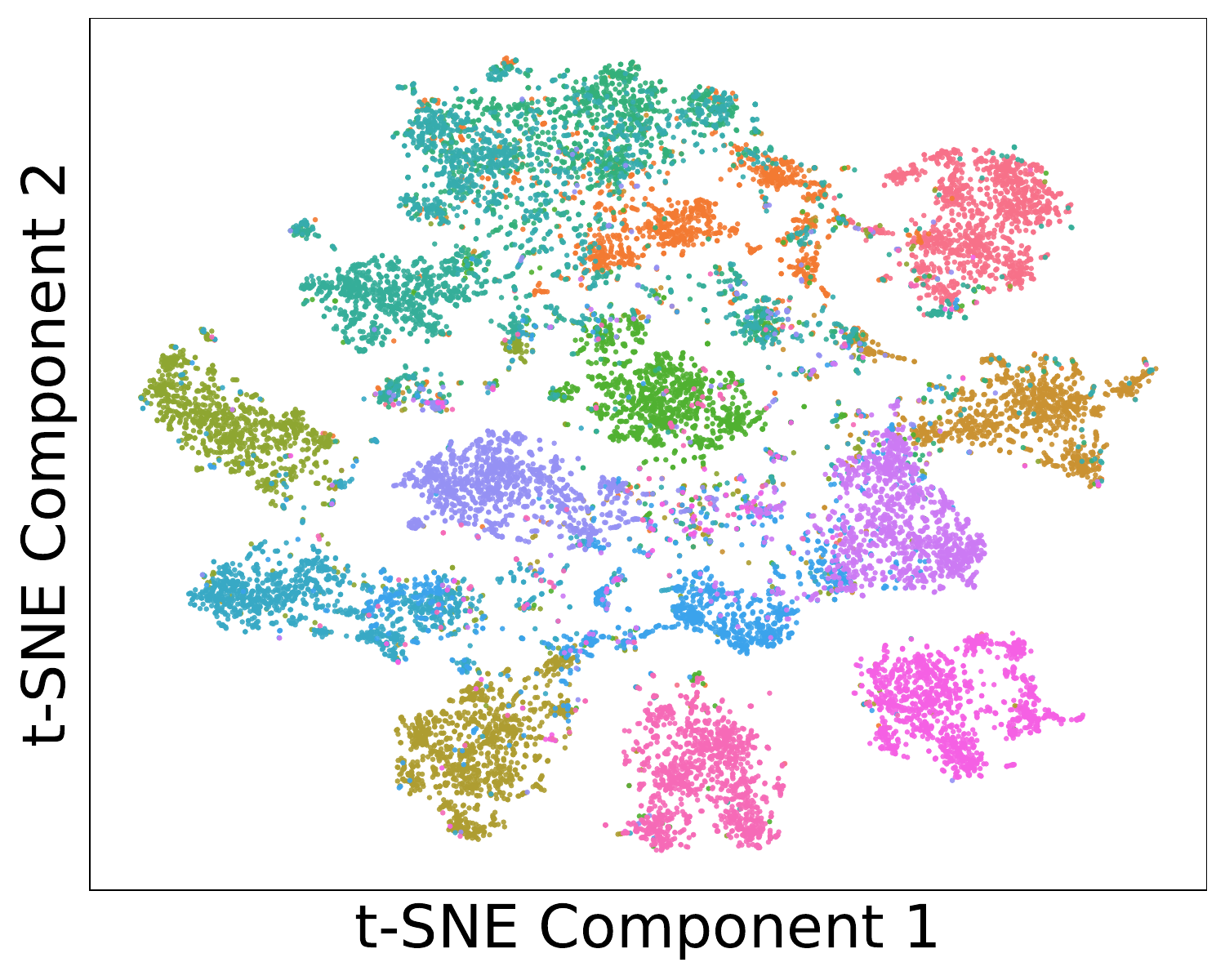}\label{fig:vis_ours}}
  \caption{Visualization of embeddings used for clustering on ImageNet-Dogs dataset. a) image embeddings from CLIP. b) noun embeddings by Eq.~\eqref{eq:noun_counterpart}. c) description embeddings by Eq.~\eqref{eq:des_counterpart}. d) embeddings by our method.}\label{fig:tsne}
  \vspace{-0.3cm}
\end{figure}

\noindent{\textbf{{Effects of top-${K}$ selection.}}
In Eq.~\eqref{eq:topC}, we select the top-$K$ nouns and descriptions of each image center to construct the semantic space. To examine the impact of the selection size, we vary top-$K$ from 1 to 30 and evaluate the performance. As shown in Figure~\ref{fig:topc_ucf_101} and Figure~\ref{fig:topc_imagenet_dogs}, using only one component (top-$K$ = 1) fails to capture sufficient semantics, resulting in suboptimal clustering. On UCF-101 (Figure~\ref{fig:topc_ucf_101}), both NMI and ACC improve as top-$K$ increases, stabilizing when a moderate number of components are used. However, further increasing top-$K$ yields diminishing returns. A similar trend is observed on ImageNet-Dogs (Figure~\ref{fig:topc_imagenet_dogs}), where excessive components introduce noise and slightly degrade performance. This highlights the need to balance semantic richness and redundancy for optimal clustering.
%  In Eq.~\eqref{eq:topC}, we select the top-$K$ nouns and descriptions of each image center to construct the semantic space. To assess the impact of the number of selected components, we vary the value of top-${C}$ from 1 to 30 and analyze the resulting performance. As illustrated in Figure~\ref{fig:topc_ucf_101} and Figure~\ref{fig:topc_imagenet_dogs}, selecting a single noun or description (top-${C}$ = 1) is insufficient to capture the complete semantics of each image center, leading to suboptimal clustering performance. On the UCF-101 dataset (Figure~\ref{fig:topc_ucf_101}), both NMI and ACC exhibit an initial improvement as top-${C}$ increases, reaching a stable level when a moderate number of nouns and descriptions are included. However, when too many components are selected, performance gains tend to plateau, indicating that an optimal balance exists between capturing semantic richness and avoiding redundancy. On the ImageNet-Dogs dataset (Figure~\ref{fig:topc_imagenet_dogs}), we observe a similar trend where performance improves with an increasing number of selected components initially, but begins to fluctuate and degrade slightly when top-${C}$ is too high. This suggests that incorporating too many nouns and descriptions can introduce unrelated or noisy semantics, which may have an adverse effect on the clustering process. 
% % Overall, our method demonstrates robustness across a typical range of top-${C}$ values, maintaining relatively stable clustering performance.

\subsection{Visualization}
We perform t-SNE~\cite{tSNE} visualizations on features computed in our approach. Figure~\ref{fig:vis_clip} shows that CLIP image embeddings exhibit only partial separation with notable overlaps. Figures~\ref{fig:vis_u} and~\ref{fig:vis_v} display noun ($\mathbf{x}^u$) and caption ($\mathbf{x}^v$) embeddings, both offering clearer separation, indicating that semantic cues improve image representation. Finally, Figure~\ref{fig:vis_ours} illustrates our fused embeddings, where clusters become well-separated with minimal overlap, demonstrating markedly enhanced discriminability.
% % To intuitively demonstrate the effectiveness of our method, 
% We perform t-SNE~\cite{tSNE} visualizations on various features calculated within our approach. Figure~\ref{fig:vis_clip}  illustrates the image embeddings encoded by CLIP. The original CLIP features exhibit some separation but have noticeable overlaps, indicating limited discriminative power for clustering.
% Figure~\ref{fig:vis_u} and Figure~\ref{fig:vis_v} present the visualizations of noun embeddings ($\mathbf{\bar{x}}^u$) and caption embeddings ($\mathbf{\bar{x}}^v$), respectively. Compared to the original CLIP embeddings, both noun and description embeddings display better separation among clusters, suggesting that the semantic features derived from nouns and descriptions provide a more distinguishable representation of the images. Finally, Figure~\ref{fig:vis_ours} shows the embeddings generated by our proposed method, which combines image, noun, and description embeddings. The clusters are well-separated, with minimal overlap, demonstrating significant improvement in the discriminability of the representations. 
% % This improved representation not only captures the visual characteristics but also enriches the semantic context, leading to better alignment with the underlying categories.

\section{Conclusion}
In this work, we propose a hierar\textbf{C}hical sem\textbf{A}ntic alignm\textbf{E}nt method for image clustering, dubbed \textbf{CAE}, which improves clustering performance without training. Our method constructs the semantic space by measuring distances between images and selected nouns and captions. Then, we adaptively fuse image, noun, and caption features according to their similarity with a semantic prototype. Extensive experiments demonstrate the effectiveness of our approach. In future work, we aim to obtain more precise descriptions for images by leveraging multimodal large language models.

\section{Acknowledgements}
This research is supported by the National Natural Science Foundation of China (No.U24B20180, No. 62576330), National Natural Science Foundation of Anhui (No.2508085MF143), and the advanced computing resources provided by the Supercomputing Center of the USTC.
% Then, we design a residual attention mechanism allowing nouns and captions to interact. Extensive experiments demonstrate the effectiveness of our approach. In future work, we aim to obtain more precise descriptions for images by leveraging multimodal large language models.

% Note that the textual information used in our method is not image-specific. 

% \bigskip
% \noindent Thank you for reading these instructions carefully. We look forward to receiving your electronic files!

\bibliography{aaai2026}

@inproceedings{JULE,
  title     = {Joint unsupervised learning of deep representations and image clusters},
  author    = {Yang, Jianwei and Parikh, Devi and Batra, Dhruv},
  booktitle = {CVPR},
  year      = {2016}
}

@inproceedings{DEC,
  title     = {Unsupervised deep embedding for clustering analysis},
  author    = {Xie, Junyuan and Girshick, Ross and Farhadi, Ali},
  booktitle = {ICML},
  pages     = {478--487},
  year      = {2016}
}

@inproceedings{DCCM,
  title     = {Deep comprehensive correlation mining for image clustering},
  author    = {Wu, Jianlong and Long, Keyu and Wang, Fei and Qian, Chen and Li, Cheng and Lin, Zhouchen and Zha, Hongbin},
  booktitle = {ICCV},
  year      = {2019}
}

@inproceedings{PICA,
  author    = {Huang, Jiabo and Gong, Shaogang and Zhu, Xiatian},
  title     = {Deep Semantic Clustering by Partition Confidence Maximisation},
  booktitle = {CVPR},
  year      = {2020}
}

@inproceedings{IDFD,
  title     = {Clustering-friendly Representation Learning via Instance Discrimination and Feature Decorrelation},
  author    = {Tao, Yaling and Takagi, Kentaro and Nakata, Kouta},
  booktitle = {ICLR},
  year      = {2020}
}

@inproceedings{SCAN,
  title        = {Scan: Learning to classify images without labels},
  author       = {Van, Gansbeke Wouter and Vandenhende, Simon and Georgoulis, Stamatios and Proesmans, Marc and Van Gool, Luc},
  booktitle    = {ECCV},
  year         = {2020},
}

@inproceedings{MiCE,
  title     = {Mice: Mixture of contrastive experts for unsupervised image clustering},
  author    = {Tsai, Tsung Wei and Li, Chongxuan and Zhu, Jun},
  booktitle = {ICLR},
  year      = {2020}
}

@inproceedings{NNM,
  title     = {Nearest neighbor matching for deep clustering},
  author    = {Dang, Zhiyuan and Deng, Cheng and Yang, Xu and Wei, Kun and Huang, Heng},
  booktitle = {CVPR},
  year      = {2021}
}

@article{SPICE,
  title     = {Spice: Semantic pseudo-labeling for image clustering},
  author    = {Niu, Chuang and Shan, Hongming and Wang, Ge},
  journal   = {IEEE Transactions on Image Processing},
  pages     = {7264--7278},
  year      = {2022},
}

@inproceedings{SIC,
  title     = {Semantic-Enhanced Image Clustering},
  author    = {Cai, Shaotian and Qiu, Liping and Chen, Xiaojun and Zhang, Qin and Chen, Longteng},
  booktitle = {AAAI},
  year      = {2023}
}

@article{WordNet,
  title     = {WordNet: a lexical database for English},
  author    = {Miller, George A},
  journal   = {Communications of the ACM},
  pages     = {39--41},
  year      = {1995},
  publisher = {ACM New York, NY, USA}
}

@inproceedings{
  TAC,
  title={Image Clustering with External Guidance},
  author={Yunfan Li and Peng Hu and Dezhong Peng and Jiancheng Lv and Jianping Fan and Xi Peng},
  booktitle={ICML},
  year={2024},
}

@inproceedings{CLIP,
  author       = {Alec Radford and
                  Jong Wook Kim and
                  Chris Hallacy and
                  Aditya Ramesh and
                  Gabriel Goh and
                  Sandhini Agarwal and
                  Girish Sastry and
                  Amanda Askell and
                  Pamela Mishkin and
                  Jack Clark and
                  Gretchen Krueger and
                  Ilya Sutskever},
  title        = {Learning Transferable Visual Models From Natural Language Supervision},
  booktitle    = {ICML},
  year         = {2021}
}

@inproceedings{AFR,
  author       = {Xingyu Zhu and
                  Shuo Wang and
                  Jinda Lu and
                  Yanbin Hao and
                  Haifeng Liu and
                  Xiangnan He},
  title        = {Boosting Few-Shot Learning via Attentive Feature Regularization},
  booktitle    = {{AAAI}},
  year         = {2024}
}

@inproceedings{Flick,
  author       = {Bryan A. Plummer and
                  Liwei Wang and
                  Chris M. Cervantes and
                  Juan C. Caicedo and
                  Julia Hockenmaier and
                  Svetlana Lazebnik},
  title        = {Flickr30k Entities: Collecting Region-to-Phrase Correspondences for
                  Richer Image-to-Sentence Models},
  booktitle    = {ICCV},
  year         = {2015}
}

@inproceedings{FangIWWSDS22,
  author       = {Alex Fang and
                  Gabriel Ilharco and
                  Mitchell Wortsman and
                  Yuhao Wan and
                  Vaishaal Shankar and
                  Achal Dave and
                  Ludwig Schmidt},
  title        = {Data Determines Distributional Robustness in Contrastive Language
                  Image Pre-training {(CLIP)}},
  booktitle    = {ICML},
  year         = {2022}
}

@inproceedings{STL,
  title     = {An analysis of single-layer networks in unsupervised feature learning},
  author    = {Coates, Adam and Ng, Andrew and Lee, Honglak},
  booktitle = {{AISTATS}},
  year      = {2011}
}

@article{CIFAR,
  title     = {Learning multiple layers of features from tiny images},
  author    = {Krizhevsky, A. and Hinton, G.},
  journal   = {Master's thesis, Department of Computer Science, University of Toronto},
  year      = {2009},
  publisher = {Citeseer}
}

@inproceedings{ImageNet_10_Dogs,
  title     = {Deep adaptive image clustering},
  author    = {Chang, Jianlong and Wang, Lingfeng and Meng, Gaofeng and Xiang, Shiming and Pan, Chunhong},
  booktitle = {ICCV},
  year      = {2017}
}

@inproceedings{ImageNet,
  title        = {Imagenet: A large-scale hierarchical image database},
  author       = {Deng, Jia and Dong, Wei and Socher, Richard and Li, Li-Jia and Li, Kai and Fei-Fei, Li},
  booktitle    = {IEEE Conference on Computer Vision and Pattern Recognition},
  pages        = {248--255},
  year         = {2009}
}

@inproceedings{DTD,
  title     = {Describing textures in the wild},
  author    = {Cimpoi, Mircea and Maji, Subhransu and Kokkinos, Iasonas and Mohamed, Sammy and Vedaldi, Andrea},
  booktitle = {CVPR},
  pages     = {3606--3613},
  year      = {2014}
}

@article{UCF101,
  author       = {Khurram Soomro and
                  Amir Roshan Zamir and
                  Mubarak Shah},
  title        = {{UCF101:} {A} Dataset of 101 Human Actions Classes From Videos in The Wild},
  journal      = {CoRR},
  year         = {2012}
}

@inproceedings{SeCu,
  author       = {Qi Qian},
  title        = {Stable Cluster Discrimination for Deep Clustering},
  booktitle    = {{ICCV}},
  year         = {2023}
}

@inproceedings{DivClust,
  author       = {Ioannis Maniadis Metaxas and
                  Georgios Tzimiropoulos and
                  Ioannis Patras},
  title        = {DivClust: Controlling Diversity in Deep Clustering},
  booktitle    = {{CVPR}},
  year         = {2023}
}

@inproceedings{RPSC,
  author       = {Sihang Liu and
                  Wenming Cao and
                  Ruigang Fu and
                  Kaixiang Yang and
                  Zhiwen Yu},
  title        = {{RPSC:} Robust Pseudo-Labeling for Semantic Clustering},
  booktitle    = {{AAAI}},
  year         = {2024}
}

@article{TCL,
  author       = {Yunfan Li and
                  Mouxing Yang and
                  Dezhong Peng and
                  Taihao Li and
                  Jiantao Huang and
                  Xi Peng},
  title        = {Twin Contrastive Learning for Online Clustering},
  journal      = {International Journal of Computer Vision},
  pages        = {2205--2221},
  year         = {2022}
}

@inproceedings{DeepClusteringPeng2016,
  title     = {Deep subspace clustering with sparsity prior.},
  author    = {Peng, Xi and Xiao, Shijie and Feng, Jiashi and Yau, Wei-Yun and Yi, Zhang},
  booktitle = {IJCAI},
  year      = {2016}
}

@article{DeepClusteringPeng2018,
  title     = {Structured autoencoders for subspace clustering},
  author    = {Peng, Xi and Feng, Jiashi and Xiao, Shijie and Yau, Wei-Yun and Zhou, Joey Tianyi and Yang, Songfan},
  journal   = {IEEE Transactions on Image Processing},
  pages     = {5076--5086},
  year      = {2018}
}

@inproceedings{IMSAT,
  title        = {Learning discrete representations via information maximizing self-augmented training},
  author       = {Hu, Weihua and Miyato, Takeru and Tokui, Seiya and Matsumoto, Eiichi and Sugiyama, Masashi},
  booktitle    = {ICML},
  year         = {2017}
}

@inproceedings{IIC,
  title     = {Invariant information clustering for unsupervised image classification and segmentation},
  author    = {Ji, Xu and Henriques, Jo{\~a}o F and Vedaldi, Andrea},
  booktitle = {ICCV},
  year      = {2019}
}

@inproceedings{Yang2017Towards,
    author = {Yang, Bo and Fu, Xiao and Sidiropoulos, Nicholas D and Hong, Mingyi},
    title = {Towards K-means-friendly Spaces: Simultaneous Deep Learning and Clustering},
    booktitle = {ICML},
    year = {2017}
}

@inproceedings{ji2017deep,
    author = {Pan Ji and Tong Zhang and Hongdong Li and Mathieu Salzmann and Ian Reid},
    title = {Deep Subspace Clustering Networks Pan},
    booktitle = {NeurIPS},
    year = {2017}
}

@inproceedings{shaham2018spectral,
    author = {Shaham, Uri and Stanton, Kelly},
    title = {SpectralNet: Spectral Clustering using Deep Neural Networks},
booktitle = {ICLR},
    year = {2018}
}

@inproceedings{CCA,
  author       = {Quanxue Gao and
                  Huanhuan Lian and
                  Qianqian Wang and
                  Gan Sun},
  title        = {Cross-Modal Subspace Clustering via Deep Canonical Correlation Analysis},
  booktitle    = {{AAAI}},
  year         = {2020}
}

@inproceedings{MaoYGY21,
  author       = {Yiqiao Mao and
                  Xiaoqiang Yan and
                  Qiang Guo and
                  Yangdong Ye},
  title        = {Deep Mutual Information Maximin for Cross-Modal Clustering},
  booktitle    = {{AAAI}},
  year         = {2021}
}

@article{tSNE,
	author = {Van der Maaten, Laurens and Hinton, Geoffrey},
	journal = {Journal of Machine Learning Research},
	title = {Visualizing data using t-SNE.},
	volume = {9},
	year = {2008}}

@inproceedings{SimCLR,
  title        = {A simple framework for contrastive learning of visual representations},
  author       = {Chen, Ting and Kornblith, Simon and Norouzi, Mohammad and Hinton, Geoffrey},
  booktitle    = {International Conference on Machine Learning},
  pages        = {1597--1607},
  year         = {2020},
  organization = {PMLR}
}

@inproceedings{MOCO,
  title     = {Momentum contrast for unsupervised visual representation learning},
  author    = {He, Kaiming and Fan, Haoqi and Wu, Yuxin and Xie, Saining and Girshick, Ross},
  booktitle = {CVPR},
  year      = {2020}
}

@inproceedings{VisDesc,
  author       = {Sachit Menon and
                  Carl Vondrick},
  title        = {Visual Classification via Description from Large Language Models},
  booktitle    = {{ICLR}},
  year         = {2023}
}

@inproceedings{VIC,
  author       = {Alessandro Conti and
                  Enrico Fini and
                  Massimiliano Mancini and
                  Paolo Rota and
                  Yiming Wang and
                  Elisa Ricci},
  title        = {Vocabulary-free Image Classification},
  booktitle    = {NeurIPS},
  year         = {2023}
}

@inproceedings{CHiLS,
  author       = {Zachary Novack and
                  Julian J. McAuley and
                  Zachary Chase Lipton and
                  Saurabh Garg},
  title        = {CHiLS: Zero-Shot Image Classification with Hierarchical Label Sets},
  booktitle    = {{ICML}},
  year         = {2023}
}

@inproceedings{Ge0G00AILZ23,
  author       = {Yunhao Ge and
                  Jie Ren and
                  Andrew Gallagher and
                  Yuxiao Wang and
                  Ming{-}Hsuan Yang and
                  Hartwig Adam and
                  Laurent Itti and
                  Balaji Lakshminarayanan and
                  Jiaping Zhao},
  title        = {Improving Zero-shot Generalization and Robustness of Multi-Modal Models},
  booktitle    = {{CVPR}},
  year         = {2023}
}

@inproceedings{GCC,
  author       = {Huasong Zhong and
                  Jianlong Wu and
                  Chong Chen and
                  Jianqiang Huang and
                  Minghua Deng and
                  Liqiang Nie and
                  Zhouchen Lin and
                  Xian{-}Sheng Hua},
  title        = {Graph Contrastive Clustering},
  booktitle    = {{ICCV}},
  year={2021}
}

@inproceedings{OT,
  author       = {Dongsheng Wang and
                  Miaoge Li and
                  Xinyang Liu and
                  Mingsheng Xu and
                  Bo Chen and
                  Hanwang Zhang},
  title        = {Tuning Multi-mode Token-level Prompt Alignment across Modalities},
  booktitle    = {NeurIPS},
  year         = {2023}
}

@inproceedings{sinkhorn,
  author       = {Marco Cuturi},
  title        = {Sinkhorn Distances: Lightspeed Computation of Optimal Transport},
  booktitle    = {{NeurIPS}},
  year         = {2013}
}

@book{villani2009optimal,
  title={Optimal transport: old and new},
  author={Villani, C{\'e}dric and others},
  volume={338},
  year={2009},
  publisher={Springer}
}

@inproceedings{APE,
  author       = {Xiangyang Zhu and
                  Renrui Zhang and
                  Bowei He and
                  Aojun Zhou and
                  Dong Wang and
                  Bin Zhao and
                  Peng Gao},
  title        = {Not All Features Matter: Enhancing Few-shot {CLIP} with Adaptive Prior
                  Refinement},
  booktitle    = {{ICCV}},
  year         = {2023}
}

@inproceedings{MartinHSPA24,
  author       = {S{\'{e}}gol{\`{e}}ne Martin and
                  Yunshi Huang and
                  Fereshteh Shakeri and
                  Jean{-}Christophe Pesquet and
                  Ismail Ben Ayed},
  title        = {Transductive Zero-Shot and Few-Shot {CLIP}},
  booktitle    = {{CVPR}},
  year         = {2024}
}

@inproceedings{TangLW0H24,
  author       = {Yuwei Tang and
                  Zhenyi Lin and
                  Qilong Wang and
                  Pengfei Zhu and
                  Qinghua Hu},
  title        = {AMU-Tuning: Effective Logit Bias for CLIP-based Few-shot Learning},
  booktitle    = {{CVPR}},
  year         = {2024}
}

@inproceedings{LFSS,
  title={Learning from Sample Stability for Deep Clustering},
  author={Zhixin Li and 
          Yuheng Jia and
          Hui LIU and 
          Junhui Hou},
  booktitle={ICML},
  year={2025}
}

@inproceedings{CC,
  title     = {Contrastive clustering},
  author    = {Li, Yunfan and Hu, Peng and Liu, Zitao and Peng, Dezhong and Zhou, Joey Tianyi and Peng, Xi},
  booktitle = {AAAI},
  year      = {2021}
}

@inproceedings{SSP,
  author       = {Xingyu Zhu and
                  Beier Zhu and
                  Yi Tan and
                  Shuo Wang and
                  Yanbin Hao and
                  Hanwang Zhang},
  title        = {Selective Vision-Language Subspace Projection for Few-shot {CLIP}},
  booktitle    = {{ACM} Multimedia},
  pages        = {3848--3857},
  publisher    = {{ACM}},
  year         = {2024}
}

@inproceedings{frolic,
  author       = {Xingyu Zhu and
                  Beier Zhu and
                  Yi Tan and
                  Shuo Wang and
                  Yanbin Hao and
                  Hanwang Zhang},
  title        = {Enhancing Zero-Shot Vision Models by Label-Free Prompt Distribution
                  Learning and Bias Correcting},
  booktitle    = {NeurIPS},
  year         = {2024}
}

@article{protomm,
  author       = {Xingyu Zhu and
                  Shuo Wang and
                  Beier Zhu and
                  Miaoge Li and
                  Yunfan Li and
                  Junfeng Fang and
                  Zhicai Wang and
                  Dongsheng Wang and
                  Hanwang Zhang},
  title        = {Dynamic Multimodal Prototype Learning in Vision-Language Models},
  journal      = {CoRR},
  volume       = {abs/2507.03657},
  year         = {2025}
}

@article{zhu2025enhancing,
  title={Enhancing CLIP Robustness via Cross-Modality Alignment},
  author={Zhu, Xingyu and Zhu, Beier and Wang, Shuo and Zhao, Kesen and Zhang, Hanwang},
  journal={arXiv preprint arXiv:2510.24038},
  year={2025}
}

@inproceedings{qiuEOC,
  title={Accelerating diffusion transformer via error-optimized cache},
  author={Qiu, Junxiang and Wang, Shuo and Lu, Jinda and Liu, Lin and Jiang, Houcheng and Zhu, Xingyu and Hao, Yanbin},
  booktitle={Proceedings of the 33rd ACM International Conference on Multimedia},
  pages={9588--9597},
  year={2025}
}

@article{likas2003global,
  title={The global k-means clustering algorithm},
  author={Likas, Aristidis and Vlassis, Nikos and Verbeek, Jakob J},
  journal={Pattern recognition},
  volume={36},
  number={2},
  pages={451--461},
  year={2003},
  publisher={Elsevier}
}

\clearpage
% \input{ReproducibilityChecklist}
% % \clearpage
% \newpage
% \cleardoublepage
\clearpage
\onecolumn

\end{document}